\documentclass{article}
\PassOptionsToPackage{numbers, compress, sort}{natbib}

    \usepackage[preprint]{neurips_2025}

\usepackage[utf8]{inputenc} 
\usepackage[T1]{fontenc}    
\usepackage{hyperref}       
\usepackage{url}            
\usepackage{booktabs}       
\usepackage{amsfonts}       
\usepackage{nicefrac}       
\usepackage{microtype}      
\usepackage{xcolor}         

\usepackage[skip=0pt]{caption}

\usepackage{amsmath}
\usepackage{amsthm}
\usepackage{amssymb}
\usepackage{subcaption}
\usepackage{wrapfig}
\usepackage{enumitem}
\usepackage{graphicx}
\usepackage{lipsum}
\usepackage{ifthen}
\usepackage{makecell}
\usepackage{xspace}
\usepackage{multirow}
\usepackage{multicol}
\usepackage{colortbl}
\usepackage{graphicx}

\usepackage{thmtools}
\usepackage{thm-restate}

\definecolor{mygreen}{RGB}{93,173,85}
\definecolor{darkred}{RGB}{204, 0, 0}
\definecolor{darkgreen}{RGB}{0, 153, 0}
\definecolor{cobalt}{rgb}{0.0, 0.28, 0.67}

\newcommand{\reshl}[2]{
\text{#1} \fontsize{4.5pt}{1em}\selectfont\color{mygreen}{$\!\uparrow\!$ \textbf{#2}}
}

\newcommand{\randomchar}{%
    \edef\case{\number\numexpr\pdfuniformdeviate 2\relax}%
    \ifcase\case
        \edef\temp{\number\numexpr\pdfuniformdeviate 26 + 97\relax}%
    \else
        \edef\temp{\number\numexpr\pdfuniformdeviate 26 + 65\relax}%
    \fi
    \symbol{\temp}%
}

\newcommand{\randname}{\texttt{RR-Cluster}\xspace}

\usepackage{threeparttable}
\makeatletter
\newcommand{\thickhline}{%
    \noalign {\ifnum 0=`}\fi \hrule height 0.6pt
    \futurelet \reserved@a \@xhline
}
\makeatother
\definecolor{mygray}{gray}{.9}
\definecolor{mygreen}{RGB}{93,173,85}
\usepackage{float}

\usepackage{multirow}
\usepackage{colortbl}
\definecolor{lightgray}{gray}{.9}
\definecolor{deepgray}{gray}{.8}
\newcolumntype{I}{!{\vrule width 1pt}}

\usepackage{tikz}
\usetikzlibrary{tikzmark,calc}

\renewcommand{\arraystretch}{0.9}

\usepackage[linesnumbered,ruled,noend]{algorithm2e}

\theoremstyle{plain}
\newtheorem{theorem}{Theorem}[section]

\newtheorem{lemma}[theorem]{Lemma}
\newtheorem{corollary}[theorem]{Corollary}
\theoremstyle{definition}
\newtheorem{definition}[theorem]{Definition}
\newtheorem{assumption}[theorem]{Assumption}
\theoremstyle{remark}

\title{Differentially Private Federated Clustering \\
with Random Rebalancing}

\author{
    Xiyuan Yang \\
    University of Illinois Urbana-Champaign \\
    \texttt{xiyuany4@illinois.edu} \\
\And
    Shengyuan Hu \\
    Carnegie Mellon University \\
    \texttt{shengyuanhu@cmu.edu} \\
\And
    Soyeon Kim \\
    KAIST \\
    \texttt{purplehibied@kaist.ac.kr} \\
\And
    Tian Li \\
    University of Chicago \\
    \texttt{litian@uchicago.edu}\\
}

\begin{document}

\maketitle

\begin{abstract}
Federated clustering aims to group similar clients into clusters and produce one model for each cluster. Such a personalization approach typically improves model performance compared with training a single model to serve all clients, but can be more vulnerable to privacy leakage. Directly applying client-level differentially private (DP) mechanisms to federated clustering could degrade the utilities significantly. We identify that such deficiencies are mainly due to the difficulties of averaging privacy noise within each cluster (following standard privacy mechanisms), as the number of clients assigned to the same clusters is uncontrolled. To this end, we propose a simple and effective technique, named \randname, that can be viewed as a light-weight add-on to many federated clustering algorithms. \randname achieves reduced privacy noise via \textit{\underline{r}andomly \underline{r}ebalancing} cluster assignments, guaranteeing a minimum number of clients assigned to each cluster. We analyze the tradeoffs between decreased privacy noise variance and potentially increased bias from incorrect assignments and provide convergence bounds for \randname. Empirically, we demonstrate that \randname plugged into strong federated clustering algorithms results in significantly improved privacy/utility tradeoffs across both synthetic and real-world datasets.
\end{abstract}

\section{Introduction}
\label{sec:intro}
Clustering is a critical approach for adapting to varying client distributions in federated learning (FL) by learning one statistical model for each group of similar clients~\citep{tian_fed_survey,kairouz2021advances}. 
It usually outperforms the canonical FL formulation of learning a single global model, as it outputs several models to serve heterogeneous client population~\citep{ifca}.  Federated clustering typically works by identifying clusters each client belongs to and incorporating model updates into the corresponding cluster models iteratively. 
However, performing clustering requires clients to transmit additional data-dependent information, making the federated learning system more vulnerable to privacy leakage. 

In this work, we aim to finally output $k$ clustered models such that the set of these $k$ models satisfies a global client-level differential privacy (DP)~\citep{dp,iclr18_dpfedavg}, with a trusted central server. Despite recent attempts privatizing federated clustering by DP~\citep{iotj_dpfedc, percom_dcfl}, we identify that a fundamental tension between DP and clustering still remains---DP tries to hide individual client's contributions to the output models, whereas clustering may expose more client information (i.e., client model updates), especially if the cluster size is small (or even one). We need much larger effective noise to privatize those clients that have a huge impact on their assigned clusters, which can degrade the overall privacy/utility tradeoffs. For instance, we observe significant utility drops if directly applying standard DP mechanisms on top of existing federated clustering algorithms (Figure~\ref{fig:intro_acc} and Section~\ref{sec:exps}).


\begin{figure}[t]
    \centering
    \begin{minipage}{0.34\textwidth}
        \centering
        \includegraphics[width=\linewidth]{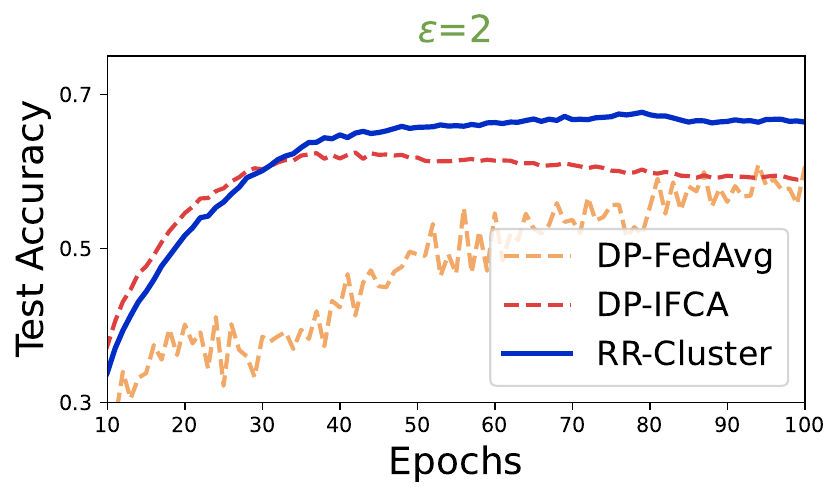}
    \end{minipage}
    \hfill
    \begin{minipage}{0.32\textwidth}
        \centering
        \includegraphics[width=\linewidth]{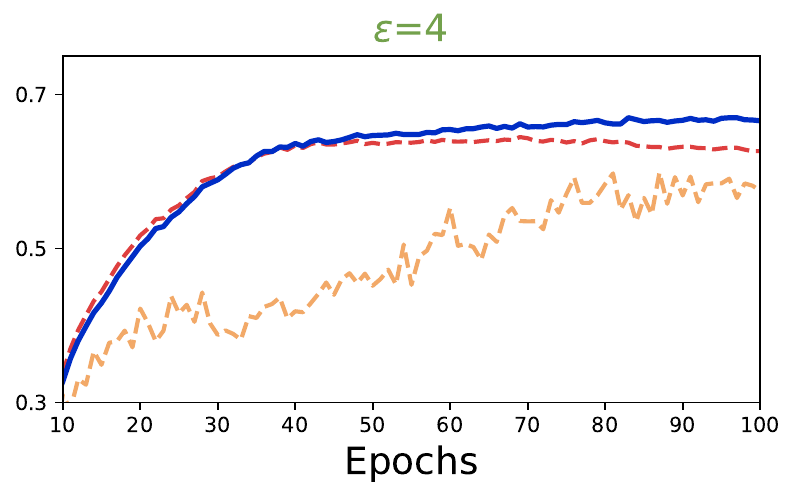}
    \end{minipage}
    \hfill
    \begin{minipage}{0.32\textwidth}
        \centering
        \includegraphics[width=\linewidth]{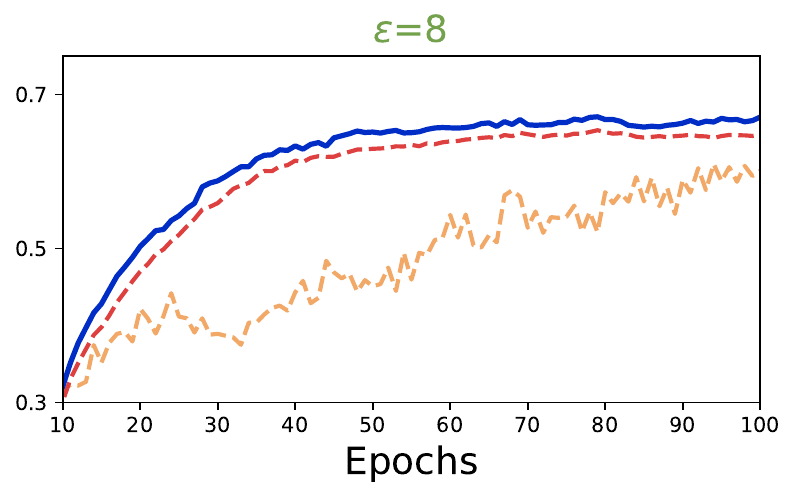}
    \end{minipage}
    \caption{\small Test accuracy on FashionMNIST~\citep{fashionmnist} with a neural network model. 
    Directly privatizing previous clustering method~\citep{ifca} (DP-IFCA) degrades model performance significantly, even underperforming the non-personalized approach (DP-FedAvg~\citep{iclr18_dpfedavg}) for some privacy budgets. The proposed \randname plugged into IFCA achieves higher accuracy in private settings under various $\varepsilon$'s.
    }
    \label{fig:intro_acc}
\end{figure}

Inspired by these insights, we propose a simple and light-weight technique, \randname, that prevents too small clusters by randomly sampling a subset of clients' model updates belonging to large clusters into small ones, thus guaranteeing each cluster having at least $B$ assigned model updates at each round. Such data-independent random rebalancing step can be applied on top of various clustering algorithms where the server aggregates model updates within each cluster based on learnt cluster assignments, without extra privacy or communication costs. 
\randname has at least $B$ client model updates within each cluster, thus resulting in a smaller effective privacy noise after model averaging. 
We note that a key hyperparameter here is $B$, the lower bound of the number of clients assigned to each cluster. When $B=0$, it recovers any base clustering algorithm. When $B$ increases, we reduce privacy noise at the cost of potentially increasing clustering bias by moving clients from correct clusters to incorrect ones. Theoretically, we analyze the effects of $B$ in the convergence bounds when plugged into a competitive clustering algorithm~\citep{ifca} in federated settings. Empirically, we find that (1)~small values of $B$ can improve privacy/utility tradeoffs significantly, (2)~and that the potential clustering bias from \randname does not hurt the learning process too much, especially in cases where the original assignment before rebalancing may not be accurate. 
Our contributions are summarized as follows.
\begin{itemize}[leftmargin=*]
    \item We propose \randname, a simple technique to improve privacy/utility tradeoffs for federated clustering by randomly sampling model updates from large clusters to small ones (Section~\ref{sec:method_0}). It can be used as a light-weight add-on on top of various federated clustering algorithms.
    \item We theoretically analyze the privacy and convergence bounds of \randname in convex settings, revealing a tradeoff between reduced privacy noise and increased bias (Section~\ref{sec:theory}).
    \item Empirically, we demonstrate the effectiveness of \randname across both real-world and synthetic federated datasets under natural and controlled heterogeneous data distributions (Section~\ref{sec:exps}). \randname 
    outperforms vanilla private baselines by a large margin. 
\end{itemize}

\section{Related Works and Background}
\label{sec:related_work}
\paragraph{Federated Clustering and Personalization.}
Federated clustering is a critical approach for federated networks by learning one model for each group of similar clients, which can benefit various downstream tasks including personalization~\citep{dennis2021heterogeneity}. 
The IFCA method \citep{ifca} is one of the most popular clustering algorithm, where each client 1) computes the losses of all cluster models on the local dataset and 2) selects the cluster having model of minimum test loss as their belonged cluster. Other clustering approaches leverage different distance measures for the cluster assignment \citep{fesem} or incorporate an additional global model to improve performance \citep{fedcam}. 
While these clustering methods significantly enhance the performance of federated systems, challenges remain to guarantee differential privacy (DP) for the entire learning pipeline. In Section~\ref{sec:exps}, we show how plugging \randname into some of these clustering methods can effectively protect data privacy while achieving better performance relative to directly privatizing these clustering methods.

\paragraph{Differentially Private Distributed/Federated Clustering.}
It is critical to ensure (differential) privacy when clustering distributed and sensitive data, which has been studied extensively in prior works mostly for data analytics~\citep[e.g.,][]{dp_cluster,dp_cluster2,optim_dp_cluster,balcan2017differentially,chang2021locally,cohen2022scalable}. 
However, joint clustering and \textit{learning} algorithms under privacy constraints are generally less explored.
In the context of federated learning, existing works have considered private federated clustering in different settings. For instance, the DCFL algorithm \citep{percom_dcfl} combines a dynamic clustering method with DP under global DP. \citet{malekmohammadi2025} perform cluster assignments based on clients' model updates and loss values. 
We note that some works may provide insufficient privacy protection without privatizing all the raw-data-dependent intermediate information in the clustering process (e.g., cluster assignment identifiers)~\citep{iotj_dpfedc,ppcfl}. 
Despite the existence of private federated clustering algorithms, issues related to high sensitivity of averaged model updates that belong to  small clusters (thus resulting in large privacy noise) still remain. 
Some works target at sample-level DP~\cite{liu2022privacy} or local DP~\citep{acsfl}, different from our setting focusing on client-level global DP with a trusted central server, as defined below.

\begin{definition}[Differential privacy~\citep{dp}] \label{def:dp}
    A randomized algorithm $\mathcal{M}: U \to \mathbb{R}^{k \times d}$ is $(\varepsilon, \delta)$-differentially private if for all adjacent datasets $D, D' \in U$ and the algorithm output $O \in \mathbb{R}^{k \times d}$, 
    \setlength{\abovedisplayskip}{5pt}
    \setlength{\belowdisplayskip}{1pt}
    $$
        \Pr(\mathcal{M}(D)\in O) \leq e^\varepsilon \cdot \Pr(\mathcal{M}(D')\in O) + \delta.
    $$
\end{definition}
\vspace{-5pt}
For client-level global DP, we define adjacent dataset $D'$ and $D$ by adding or removing any single client, which is a commonly-used privacy notion in federated settings~\citep{iclr18_dpfedavg}. In the clustering setting, we consider the output $\mathcal{M}(D) \in \mathbb{R}^{k \times d}$ to be $k$ clustered models, each in a $d$ dimensional space. 

Additionally, we mainly adopt  R\'enyi differential privacy (RDP)~\citep{renyi_dp} for privacy accounting. RDP is a generalization of DP based on  R\'enyi divergence, which offers an easier accounting process than $(\epsilon,\delta)$-DP~\citep{renyi_dp}. For completeness, we introduce a formal definition of RDP in Definition~\ref{def:rdp}, and restate existing RDP-to-DP conversion theorem in Theorem~\ref{thm:rdp_to_dp} in the Appendix.
\begin{definition}
\label{def:rdp}
    ($(\lambda,\epsilon)$-R\'enyi Differential Privacy~\citep{renyi_dp}) A randomized mechanism $\mathcal{M}: U \to \mathbb{R}^{k \times d}$ is  $(\lambda,\epsilon)$-RDP if for all adjacent datasets $D, D' \in U$ and the algorithm output $O \in \mathbb{R}^{k \times d}$,
    $$
        D_{\lambda}(\mathcal{M}(D)||\mathcal{M}(D'))=\frac{1}{\lambda-1}\log \mathbb{E}_{o\sim \mathcal{M}(D)}\left[\left(\frac{\Pr(\mathcal{M}(D)=o)}{\Pr(\mathcal{M}(D')=o)}\right)^{\lambda-1}\right] \leq \epsilon.
    $$
\end{definition}

\section{Differentially Private Federated Clustering with Random Rebalancing} \label{sec:method_0}

In this section, we start with the problem formulation of federated clustering (Section~\ref{sec:method:formulation}). Then, we introduce the general framework of \randname and present an instantiation of \randname by plugging it into a popular clustering method IFCA~\citep{ifca} (Sections~\ref{sec:method} and~\ref{subsec:ours_ifca}). Finally, we discuss tradeoffs introduced by \randname around privacy noise reduction and  clustering bias (Section~\ref{sec:method:tradeoff}). 

\subsection{Problem Formulation} \label{sec:method:formulation}

In this work, we study the common federated clustering objective, i.e., grouping $M$ clients into $k$ groups and outputting one model for each group. 
We assume there are $k$ underlying data distributions $\{\mathcal{D}^0, {\cdots}, \mathcal{D}^{k-1}\}$, and the $M$ client indices are partitioned into $k$ disjoint sets $\{S_0, {\cdots}, S_{k-1}\}$ indicating which clusters (distributions) clients belong to. Local data $D_i$ of client $i \in [M]$ follow the distribution $\mathcal{D}^j$ if $i \in S_j$.
Let $f(\theta; z)$ be the loss function with model parameter $\theta \in \mathbb{R}^{d}$ and data point $z$. The optimization objective of cluster $j$  is:
\setlength{\abovedisplayskip}{4pt}
\setlength{\belowdisplayskip}{4pt}
\begin{equation}
    F^j(\theta) :=\mathbb{E}_{z \sim \mathcal{D}_j}[f(\theta; z)], \quad j \in [k]. \nonumber
\end{equation}
We search for optimal parameters $\hat\theta_j^*$ for each cluster $j {\in} [k]$ as
    $\hat\theta_j^* :=\mathop{\arg\min}\limits_{\hat\theta} F^j(\hat\theta)$.
For each client $i \in [M]$, we define their local empirical loss $F_i(\theta)$ as 
    $F_i(\theta) := \frac{1}{|D_i|} \sum_{z \in D_i} f(\theta; z)$.
In practice, $F^j (\theta)$ is often approximated by the empirical version
$
    {F}^j(\theta) := \frac{1}{|S_j|} \sum_{ i \in S_j} F_i(\theta).
$
For federated clustering algorithms, the server typically maintains cluster models $\{\hat{\theta}_j\}_{j \in [k]}$, which are periodically aggregated from corresponding client models based $\{S_j\}_{j \in [k]}$. 
Different clustering algorithms can employ different methods to iteratively map clients to their corresponding clusters. For instance, IFCA~\cite{ifca} uses loss values as proxies for such assignment.

\subsection{\randname Algorithm}
\label{sec:method}
\vspace{-0.05in}

We introduce the differentially private federated clustering framework and the proposed \randname method. As \randname is compatible with many clustering methods, we present a general algorithm to illustrate the main idea in this section and present an instantiation with IFCA \citep{ifca} in Section~\ref{subsec:ours_ifca}. 

At a high level, federated clustering jointly learns $k$ cluster model parameters and updates cluster assignments of clients, as summarized in Algorithm~\ref{alg:pseudo-code-general}. In the $t$-th round, server samples $qM$ clients, and sends the current $k$ cluster models to each of them.
Each client $i$ identifies the cluster they belong to and optimizes that cluster model's parameters using local data. They then send the model updates $\Delta \theta_i^t$ to the server as well as other information necessary for clustering (Line~\ref{op:sendback}). Without our proposed plug, the server would identify model updates corresponding to each cluster (Line~\ref{op:alg1_updateS}), privatize and aggregate those updates within each cluster (Line~\ref{op:server_clip_theta} to Line~\ref{op:update_cluster}), and output new cluster models $\{\hat{\theta}_j^{T}\}_{j\in[k]}$. This procedure directly privatizes existing federated clustering algorithms.

\begin{algorithm}[t]
\SetAlgoLined
\DontPrintSemicolon
\SetNoFillComment
\setlength{\abovedisplayskip}{1pt}
\setlength{\belowdisplayskip}{1pt}
\setlength{\abovedisplayshortskip}{1pt}
\setlength{\belowdisplayshortskip}{1pt}
\begin{tikzpicture}[remember picture, overlay]
        \draw[line width=0pt, draw=red!30, rounded corners=2pt, fill=red!20, fill opacity=0.3]
            ([xshift=40pt,yshift=3pt]$(pic cs:a)+(338pt,-110pt)$) rectangle
            ([xshift=-5pt,yshift=0pt]$(pic cs:b)+(17pt,-145pt)$);
\end{tikzpicture}
\KwIn{total communication round $T$, number of clusters $k$, number of clients $M$, init clusters model $\{\hat{\theta}_j^0\}_{j \in [k]}$, 
client sampling rate $q$, parameter clipping bound $C_\theta$, rebalancing threshold $B~(1 \leq B \leq \frac{qM}{k})$}

\caption{The Proposed Method: \randname}
\label{alg:pseudo-code-general}
    \For{$t=0, \cdots, T-1$}{

        {S}erver samples a subset of clients $M^t$ with probability $q$ 
        \label{op:sample}

        Server sends $\{\hat{\theta}_j^t\}_{j \in [k]}$ to sampled clients
        \label{op:send}
        
        {C}lient $i \in M^t$ sends back model updates $\Delta {\theta}_i^t$ and other information ${s}_i$
        \label{op:sendback}

        Server privatizes  $\{s_i\}_{i \in M^t}$ into
        $\{\widetilde{s}_i^t\}_{i \in M^t}$  if $\{s_i\}_{i \in M^t}$ directly queries raw client data

        Server partitions $M^t$ into $k$ groups  (with client indices in $\{S_j^t\}_{j \in [k]}$) based on $\widetilde{s}_i$ 
        \label{op:alg1_updateS}

        Server determines large/small groups comparing $|S_j^t|$ with $B$, $j \in [k]$


        Server {samples} clients from large groups to insert into small ones, i.e.,  rebalancing cluster assignments s.t. $|S_j^t| \geq B$ for all $j\in [k]$
        \label{op:alg1_rebalance}

        Server clips client updates 
        $\Delta\bar{\theta}_i^t \gets \frac{\Delta\theta_i^t}{\max (1,\lVert s\lVert/C_\theta)}, ~{i\in M^t}$ \label{op:server_clip_theta}
        
        Server aggregates client updates
        $\Delta\widetilde{\theta}_j^t \gets \frac{1}{|S_j^t|}\left(\sum_{i \in S_j^t}\Delta\bar{\theta}^t_{i}+\mathcal{N}(0,(2C_\theta)^2\sigma_\theta^2)\right), ~{j \in [k]}$ \label{op:server_agg}
        
        Server uses noisy updates to update cluster models $ \hat{\theta}_j^{t+1} \gets \hat{\theta}_j^{t} + \gamma \Delta\widetilde{\theta}_j^t, ~{j\in [k]}$
        \label{op:update_cluster}
        
    }
 
    \Return{cluster models $\{\hat{\theta}_j^{T}\}_{j\in[k]}$}
\vspace{-2pt}
\end{algorithm}

However, one fundamental challenge associated with these clustering algorithms is that the cardinality of the set of model updates corresponding to each cluster  (i.e., $|S_j^t|, j \in [k]$) is not controlled. Therefore, the privacy noise needed to  hide client model updates within each cluster (i.e., effective noise added after averaging model updates) could be large. To address this issue, we propose a simple and effective plug-in (highlighted in \textcolor{red}{red}) that uniformly randomly samples model updates from large groups (i.e., $|S_j^t| > B$) and move them to small ones. We guarantee that each cluster has a minimum of $B$ client model updates ($1 \leq B \leq {qM}/{k}$), reducing privacy noise by averaging across at least $B$ updates to update $\hat{\theta}_j^t$.  The resulting private clustering framework is named \randname. 
Note that we can set $B$ so that any client update can only be sampled and reassigned to another cluster at most once.
Note that given the modularity of \randname, it can preserve the communication-efficiency of any federated clustering method that it is plugged into.
Furthermore, we discuss the tradeoffs introduced by the rabanlancing threshold $B$ in Section~\ref{sec:method:tradeoff}.

\subsection{\randname Plugged into IFCA}
\label{subsec:ours_ifca}
\vspace{-5pt}

Having established the high-level structure of \randname, we now plug the random rebalancing idea into one popular clustering method IFCA~\citep{ifca} as an example, presented in Algorithm~\ref{alg:pseudo-code}. We also empirically evaluate the performance of \randname plugged into other federated clustering algorithms in Section~\ref{sec:exps}.  
In IFCA, at each iteration, each client independently determines the cluster they belong to by selecting the model with the smallest loss evaluated using local data, as shown in Line~\ref{op:select_cluster}. 
Each selected client $i$ additionally communicates their cluster identifier to the server organized as a one-hot vector $s_i\in\mathbb{R}^k$, which is privatized along with model parameters. 

\setlength{\textfloatsep}{9pt}
\begin{algorithm}[!t]
\SetAlgoLined
\DontPrintSemicolon
\SetNoFillComment
\setlength{\abovedisplayskip}{0pt}
\setlength{\belowdisplayskip}{0pt}
\setlength{\abovedisplayshortskip}{0pt}
\setlength{\belowdisplayshortskip}{0pt}
\begin{tikzpicture}[remember picture, overlay]
        \draw[line width=0pt, draw=red!30, rounded corners=2pt, fill=red!20, fill opacity=0.3]
            ([xshift=40pt,yshift=3pt]$(pic cs:a)+(343pt,-125pt)$) rectangle
            ([xshift=-5pt,yshift=0pt]$(pic cs:b)+(17pt,-183pt)$);
\end{tikzpicture}
\KwIn{The same as Algorithm~\ref{alg:pseudo-code-general}}
\caption{\randname (IFCA)}
\label{alg:pseudo-code}

    \For{$t=0, \cdots, T-1$}{

        Same steps as in Algorithm~\ref{alg:pseudo-code-general}, Line \ref{op:sample} - Line \ref{op:send}
                
        \For{client $i \in M^t$ in parallel}{
        
            Determines its cluster: $\hat{j}$ $\gets$ $\underset{j\in [k]}{\arg\min}~F(\hat{\theta}^t_j, D_i)$ 
            \label{op:select_cluster}
            and embeds $\hat{j}$ into one-hot: $s_i^t$ $\gets$ $\mathbf{1}_{\{j=\hat{j}\}}$ 
            
            $\theta^t_{i}\gets \text{Local Training}(D_i, \hat{\theta}^t_{\hat{j}})$, ~ $\Delta \theta^t_{ i} \gets \theta^t_{ i} - \hat{\theta}_{\hat{j}}^t$ 
            \label{op:get_update}

            Send $\Delta{\theta}^t_{i}$ and ${s}_i^t$ back to server
        }

    Server privatizes $s_i$ as $\widetilde{s}_i^t \gets \frac{s_i^t}{\max (1,\lVert s\lVert/C_s)}+\mathcal{N}(0,C_s^2\sigma_s^2)$ for ${i\in M^t}$ 
    
    Server updates cluster assignments: 
    $S_j^t.\texttt{append(}i\texttt{)} ~\text{where~} {j} \gets \underset{j\in [k]}{\arg\max}~\widetilde{s_i}[j]$ for $ i \in M^t$
    
    Server determines large/small groups 
    $j_l = \{j \big{|}\left|S_j^t\right|\geq B, j \in [k]\};~j_s = [k] \setminus j_l$ 
    \label{op:determine_lm}

     Server samples client indices assigned to large clusters $\bigcup_{j \in j_l} \{S_j^t\}$ uniformly at random to expand small ones $\{S_j^t\}_{j\in j_s}$ such that $\left|S_j^t\right| = B$ for $j \in j_s$

    Same privatizes $\Delta\theta_i^t$ and update cluster model $\hat{\theta_j^t}$ as in Algorithm~\ref{alg:pseudo-code-general}, Line \ref{op:server_clip_theta} - Line \ref{op:update_cluster}
        
 }
    \Return{cluster models $\{\hat{\theta}_j^{T}\}_{j\in[k]}$}
    \label{op:return}
    \vspace{-2pt}
\end{algorithm}

\paragraph{Other Differentially Private Mechanisms.}
There are various privatization algorithms for model parameters $\theta$ and cluster identifier $s_i$'s. \randname is agnostic of the specific choice of private mechanisms. For $\theta \in \mathbb{R}^{d}$, it is a common practice to use the Gaussian mechanism after clipping to upper bound the $L_2$ norm of parameters. However, for cluster identifiers $s_i^t$, which is a $k$-dimensional one-hot discrete vector, it is also reasonable to adopt the exponential mechanism. 
Though we use the Gaussian mechanism in experiments (Section~\ref{sec:exps}), \randname achieves better performance regardless of how to  privatize identifiers, as it reduces  DP noise for model parameters. 

\subsection{Effects of $B$: Tradeoffs Between Privacy Noise  
and Clustering Bias} \label{sec:method:tradeoff}
Our method mainly benefits from the reduction of privacy noise after inserting model updates sampled from large groups $\{S_j\}_{j \in j_l}$ into the small ones $\{S_j\}_{j \in j_s}$ (Line 11). However, this approach may introduce bias if the updates that are initially correctly assigned to large groups are incorrectly re-assigned to one of the small ones. 
We discussed this issue both theoretically (Section~\ref{sec:tradeoff}) and empirically through a case study on synthetic data (Appendix~\ref{app:side_effect}). Empirically, we observe that benefits of reduced privacy noise outweigh potential bias. 

Additionally,  we note that the simple random rebalancing step may not introduce much error or bias, as the initial cluster assignment (before rebalancing) may be incorrect anyway, which is a common issue for (federated) clustering sometimes known as model collapse~\cite{motley}.
That is, the clustering algorithm starts with $k$ models but only a subset of them get effectively trained and the solutions get gradually biased towards those useful models. As empirically demonstrated in Table~\ref{tab:synthetic_additional}, a subset of models can collapse after training. 
As our method maintains a minimum number of updates assigned to each cluster, one side effect is that it ensures each cluster receives some updates and thus would not be under-trained even if the updates may be slightly biased. 
Though \randname is designed for private training, in Section~\ref{sec:exps}, we show that our method remains competitive or even outperforms other baselines in non-private training as well, due to this side effect.

\section{Theoretical Analysis} \label{sec:theory}

\subsection{Privacy Analysis}
\label{sec:privacy_analysis}
We state the privacy guarantees for the proposed method Algorithm~\ref{alg:pseudo-code-general} in terms of user-level global DP and RDP. In each communication round, conditioned on a randomly sampled subset of clients, there are two other random processes: (1) privatizing cluster identifiers and (2) privatizing model parameters. Hence, the privacy loss per round is accumulated from the composition of these two processes and the client subsampling step. First, we have the following proposition.

\begin{restatable}{proposition}{rdpsingleround}
    At each  round of \randname, given a subset of selected clients $M^t$ as input, a randomized mechanism $\mathcal{M}$ that outputs $k$ new clustered models satisfies $(\alpha, \varepsilon_1{+}\varepsilon_2)$-RDP if: (1) we set the noise scale for model parameters to $\sigma_\theta {=} \sqrt{\frac{\alpha}{2\varepsilon_2}}$ and (2) we use any cluster identifier privatization method that satisfies $(\alpha, \varepsilon_1)$-RDP. 
\label{thm:rdp_single_round}
\vspace{-0.05in}
\end{restatable}

The above result directly follows from adaptive composition properties of RDP~\citep{renyi_dp} combined with sensitivity bound of the sum of model updates within each cluster (Appendix~\ref{app:sensitivity}). As discussed in Section~\ref{sec:method}, we can adopt existing randomized algorithms to privatize cluster identifiers, as long as it achieves $(\alpha, \varepsilon_1)$-RDP. 
Using these noisy cluster identifiers, we identify large and small clusters and perform random rebalancing. Given rebalanced clusters (i.e., disjoint partitions of selected clients), we privatize each cluster model with a Gaussian mechanism to obtain $(\alpha, \varepsilon_2)$-RDP for each model. 
Since these $(\alpha, \varepsilon_2)$-RDP models are trained on disjoint partition of clients in parallel, we can achieve overall $(\alpha, \varepsilon_2)$-RDP for the output of $k$ models per round conditioned on subsampled clients.

To derive the final privacy guarantee of \randname, we analyze (1) the client sampling process and (2) the composition of $T$ communication rounds. Note that the randomized mechanism $\mathcal{M}$ in Theorem~\ref{thm:rdp_single_round} takes as input the subset of clients $M^t$. Thus, given the input of total $M$ clients, we can exploit randomness during client selection and derive a stronger privacy guarantee.
Consider  $\mathcal{M}$ as a black-box mechanism, we can apply privacy amplification theorem of RDP under sampling without replacement~\citep{subsample_rdp}.
We then apply composition theorem~\citep{renyi_dp} to account for total $T$ communication rounds and derive the final RDP privacy guarantee as follows.

\begin{restatable}{theorem}{rdpTrounds}
\label{thm:rdp_T_rounds}
    Suppose a randomized mechanism $\mathcal{M}$ that inputs a subset of selected clients $M^t$ and  outputs $k$ new clustered models satisfies $(\alpha, \varepsilon_1(\alpha){+}\varepsilon_2(\alpha))$-RDP for a single communication round of \randname. Given a client sampling ratio $q$ and total $T$ communication rounds, \randname algorithm satisfies $(\alpha, \varepsilon(\alpha))$-RDP where
\label{thm:overall_privacy}
\end{restatable}
\vspace{-15pt}
\begin{small}
    \begin{align*}
    \begin{split}
        \varepsilon(\alpha) \leq \frac{T}{\alpha - 1} \cdot \log \biggl( 1 & + q^2 \binom{\alpha}{2} \min \Bigl\{ 4(e^{\varepsilon_1(2) + \varepsilon_2(2)} - 1), \quad e^{\varepsilon_1(2) + \varepsilon_2(2)} \min \{ 2, (e^{\varepsilon_1(\infty)+\varepsilon_2(\infty)} - 1)^2 \} \Bigr\} \\
        &+ \sum_{j=3}^{\alpha} q^j \binom{\alpha}{j} e^{(j-1)(\varepsilon_1(j)+\varepsilon_2(j)}) \min \{ 2, (e^{\varepsilon_1(\infty)+\varepsilon_2(\infty)} - 1)^j \} \biggr).\\
    \end{split}
    \end{align*}
\end{small}
\vspace{-8pt}

Here, $\varepsilon$ is denoted as $\varepsilon(\alpha)$, as $\varepsilon$ is functionally determined by the choice of $\alpha$. We note that the above complex bound directly follows from \citet{subsample_rdp}, which provides a precise non-asymptotic bound rather than relying on an asymptotic bound. In Appendix~\ref{app:priv_guarantee_rdp}, we provide a detailed proof of Theorem~\ref{thm:overall_privacy} and the privacy guarantee in terms of DP using RDP-DP conversion~\citep{renyi_dp}.

\subsection{Convergence of \randname}
In this section, we provide convergence guarantees of our proposed Algorithm~\ref{alg:pseudo-code} (\randname plugged into IFCA~\cite{ifca}) and theoretically explore the tradeoffs between increased clustering bias and reduced privacy noise as a result of random rebalancing with a parameter $B$. 
For simplicity, we assume that all clients participate in every round. We first state some assumptions.

\begin{assumption}
\label{ass:convex_smooth}
    For  every  $j \in [k]$, $F^j(\cdot)$ is $\lambda$-strongly convex and $L$-smooth.
\end{assumption}

\begin{assumption}
    For every $\theta$ and every $j  \in [k]$, the variance of $f(\theta ; z)$ is upper bounded by $\eta^2$, 
     when $z$ is sampled from $\mathcal{D}_j$, i.e.,  $\mathbb{E}_{z \sim \mathcal{D}_j}\left[\left(f(\theta, z)-F^j(\theta)\right)^2\right] \leq \eta^2$.
\end{assumption}

\begin{assumption} 
\label{assum:bounded_grad_variance}
    For every $\theta$ and every $j\in[k]$, the variance of $\nabla f(\theta;z)$ is upper bounded by $v^2$, where $z$ is sampled from $\mathcal{D}_j$, i.e., $\mathbb{E}_{z\sim\mathcal{D}_j} \left[\left\|\nabla f(\theta, z)-\nabla F^j(\theta)\right\|^2\right] \leq v^2$.
\end{assumption}

\begin{assumption}
    For every cluster $S_j~(j\in[k])$ in any round, the variance of cluster sizes $|S_j|$ is upper bounded by $\mu^2$, i.e., $\mathbb{E}[(|S_j|-\frac{1}{k}\sum_{j\in [k]}{|S_j|})^2] \leq \mu^2$.
\end{assumption}

\begin{assumption}
\label{ass:bounded_stochastic_gradient}
    For every $\theta$ and every $j \in [k]$, the norm of the stochastic gradient $\nabla f(\theta; z)$ is bounded, i.e., there exists a constant $G > 0$ such that $ \|\nabla f(\theta; z)\| \leq G, \forall z \sim \mathcal{D}_j. $
    \vspace{-5pt}
\end{assumption}

Here, we further assume that the adopted DP clipping threshold $C_\theta$ is larger than any stocastic gradient norm, i.e., $C_\theta > G$, the clipping does not affect model update in our analysis. This assumption is commonly used across various differential privacy analysis \citep{adadps, dpfedsam, fedavg_convergence}.

\begin{lemma}
\label{lemma:mis_cls}
Let client $i$ belong to $S^t_j$ by ground-truth at the $t$-th communication round ($t \in [T]$). Assume that we obtain the cluster model $\theta_j^t$ such that there exist $0<\beta<\frac{1}{2}$ and $\Delta :=  \min_{j \neq j'}\|\theta_j^*-\theta_{j'}^*\|$ that satisfy 
$\|\theta^t_j-\theta_j^*\| < (\frac{1}{2}-\beta) \sqrt{\frac{\lambda}{L}}{\Delta}$. 
Suppose we add zero-mean Gaussian noise with variance $\sigma_s^2$ to cluster identifiers with clipping bound 1.
Then at the $t$-th round, the probability of misclassifying client $i$'s update into any other cluster $j' \neq j$ is upper bounded by  
\begin{equation}
\begin{aligned}
\label{eq:tau}
\tau = \frac{8\eta^2k}{\beta^2 \lambda^2 \Delta^4} + \frac{\sigma_sk}{\sqrt{\pi}}\exp(-1/4\sigma_s^2) + \frac{k^2\mu^2}{(M/k-B)^2}.
\end{aligned}
\end{equation}
\end{lemma}

\vspace{-5pt}
We note that in practice, the total number of clients is much larger than the number of clusters, i.e., $M/k \gg 1$. Hence $\tau$ can be much smaller than 1. In addition, $\tau$ decreases with the increase of the model separation parameter $\Delta$, which is expected.
We prove Lemma~\ref{lemma:mis_cls} in Appendix~\ref{app:err_cluster_prob}. 

\begin{assumption}
\label{assum:seperation}
    We assume that there exists $\alpha_0 ~{\in}~ (0, \frac{1}{2})$ s.t. the initial parameters $\hat\theta_j^0$ satisfy $\|\hat\theta_j^0 - \hat\theta_j^*\| {\leq} (\frac{1}{2}-\alpha_0)\sqrt{\frac{\lambda}{L}}\Delta$ for every $j \in [k]$, where 
    $
        \Delta :=  \min_{j \neq j'}\|\theta_j^*-\theta_{j'}^*\| \geq \mathcal{O}\left(
        \alpha_0^{\frac{1}{2}}k^{\frac{3}{4}} M^{-1} B^{-\frac{1}{2}} \right),
    $
    $M > \left(\frac{B\Delta}{2\tau} {-} \frac{\rho}{\tau \gamma L}\right)$, and $\rho := M{-}(k{-}1)B$.
\end{assumption}
Assumption~\ref{assum:seperation} requires that the initialization models are good in the sense that they are close to the optimal cluster models. The condition on $\Delta$ indicates well-separated underlying cluster models. We note that the assumption on $M$ can be satisfied when the client population is large.

\begin{theorem}
\label{thm:one_step}
Let Assumptions~\ref{ass:convex_smooth}-\ref{assum:seperation} hold.  Assume in a certain iteration $t$ of Algorithm~\ref{alg:pseudo-code}, we obtain $\theta_j$ such that $\|\theta_j-\theta_j^*\| < (\frac{1}{2}-\beta) \sqrt{\frac{\lambda}{L}}{\Delta}$ for some $\beta \in (0, \frac{1}{2})$. 
Denote $\theta_j^+$ to be the parameter in next iteration. Then, for any $j \in [k]$, with probability at least $1-\delta_c$, we have 
\vspace{-1pt}
\begin{small}
\begin{equation}
\begin{aligned}
\label{eq:one_step}
\|\theta^+_j & -\theta^*_j\|\leq \left(1-\frac{\gamma L (B-{2\tau M}/{\delta_c})}{2\rho}\right)\|\theta_j - \theta^*_j\| + \epsilon, \\
\epsilon & = \frac{4 v}{\sqrt{\delta_c(B\delta_c-4\tau M)}}
+\frac{6 \tau \gamma L\Delta M}{\delta_c B}
+\frac{8\gamma v k\sqrt{\tau k M}}{B\delta_c\sqrt{\delta_c}} +\frac{2\gamma\sigma_\theta C_\theta}{B}\sqrt{d + 2\sqrt{d\ln(\tfrac{4}{\delta_c})} + 2\ln(\tfrac{4}{\delta_c})}.
\end{aligned}
\end{equation}
\end{small}
\end{theorem}
We prove Theorem~\ref{thm:one_step} in Appendix~\ref{app:conv}.
We note that this bound gets worse as the privacy noise $\sigma_{\theta}$ and $\sigma_s$ (embedded in $\tau$ in Eq.~\eqref{eq:tau}) increase. Additionally, note $\tau \propto \mathcal{O}(\Delta^{-4}M^{-2})$~(Eq.~\eqref{eq:tau}), then the error term $\epsilon$ with $\tau$ in Eq.~\eqref{eq:one_step} can be small if we have enough clients (large $M$) and a good cluster separation initialization (large $\Delta$). 
Further, we provide the final convergence results in  Corollary~\ref{conv_cor} in the appendix based on Theorem~\ref{thm:one_step}. We discuss the effects of $B$ in the next paragraph.

\label{sec:tradeoff}
\paragraph{Discussions on the Bias/Variance Tradeoff.} 
Our method aims to achieve better tradeoffs between clustering bias and privacy noise variance by random rebalancing parameterized by $B$. For clustering bias, for instance,  Lemma~\ref{lemma:mis_cls} indicates that the probability of incorrect cluster assignments $\tau$ is proportional to $\frac{c'}{(M/k-B)^2}$ for some constant $c'$. Since $M>kB$, we have that $\tau$ increases as the increase of $B$.  
This is also aligned with the intuition that a larger $B$ means sampling more client updates into potentially wrong ones. 
Considering variance, 
the error stemmed from privacy noise  (the last term in $\epsilon$) decreases as $B$ increases. Therefore, we observe that as $B$ increases, we get potentially biased cluster assignments, but significantly much smaller privacy noise perturbation, thus potentially achieving a better tradeoff between bias and variance.  We closely analyze this behavior on a synthetic dataset in Appendix~\ref{app:side_effect}.
In Section \ref{sec:exps}, we empirically show that \randname improves over the baselines under various choices of $B$ (e.g., Table~\ref{table:acc_cluster}).

\section{Experiments} \label{sec:exps}

In this section, we evaluate \randname across diverse datasets and scenarios to demonstrate its superior performance. We present main results in Section~\ref{sec:exp_sota} on both image and text tasks, showing that our approach outperforms state-of-the-art DP clustering baselines in terms of privacy/utility tradeoffs and clustering quality. 
We also demonstrate the compatibility of \randname as a general plug-in to various federated clustering algorithms. We  provide diagnostic analysis in Section~\ref{sec:diagnostic}. 

\subsection{Experimental Settings}
\label{sec:exp_setting}
\paragraph{Datasets.} We evaluate our method on four image datasets, one text dataset, and one synthetic dataset as follows. We use FashionMNIST \citep{fashionmnist} that consists of 10-class 70,000 grayscale images of clothes, and EMNIST \citep{emnist}, an extended dataset of the original MNIST, containing 814,255 images of handwritten characters with 47 categories covering both digits and letters. 
For the text data, we use the Shakespeare dataset \citep{shakes1, shakes2} on a next-character prediction task over the texts of Shakespeare's plays and writings. To control data separability and illustrate \randname's behavior in non-private settings in more detail, we generate synthetic datasets for a linear regression task. Further, we also conduct experiments on both CIFAR10 and CIFAR100~\citep{cifar} datasets; see Appendix~\ref{app:cifar} for details.

To simulate different data distributions in real-world scenarios, we explore both manual partitions with a clear clustering structure as in previous works and partitions without clear cluster priors from previous FL benchmarks. For FashionMNIST, we followed the setting in IFCA \citep{ifca}, using image rotations to simulate the underlying ground truth clusters. We partition EMNIST based on a Dirichlet distribution over the labels \citep{TFF}, and partition the Shakespeare dataset based on speaking characters in the play~\citep{leaf}. In both image tasks, we have 1000 clients in total; for the Shakespeare dataset, we use 300 clients due to limitations of speaking roles. For all experiments, we use client sample rate $q=0.1$, and tune $k\in\{2,4\}$ based on validation data.

\paragraph{Baselines.} We compare our methods with several competitive FL and federated clustering methods under DP: FedAvg \citep{fedavg}, FedPer \citep{fedper}, IFCA \citep{ifca}, FeSEM \citep{fesem} and FedCAM \citep{fedcam}. 
FedAvg is a classical FL baseline learning a global Model. 
IFCA, FeSEM and FedCAM are strong federated clustering approaches. While we focus on clustering (outputting $k$ models for $M$ clients where $k \ll M$) as opposed to broad federated personalization, we still compare with FedPer, a  personalization method as a reference point. We directly privatize these methods using standard subsampled Gaussian mechanisms guarantee the same client-level global DP across all approaches.
For clustering methods, we privatize cluster identifiers similarly as in \randname. We do not compare with other personalization works that output $M$ models for $M$ clients.

\subsection{Comparison with Strong Baselines} 
\label{sec:exp_sota}

We present experiments comparing our method with other state-of-the-art methods on FashionMNIST, EMNIST, and Shakespeare datasets. For a fair comparison, we set the same hyperparameter (e.g., number of clients, local learning rates, local iterations, and communication rounds) for all methods. 
For DP settings, we tune the clipping bound to get optimal value for each task and each method separately. See Appendix~\ref{app:exp_details} for additional details on experimental setup.


\paragraph{Plugging \randname to Base Clustering Methods.} \label{sec:fashionmnist_text}
In Table~\ref{table:acc_fashionmnist}, we use FashionMNIST rotation to create underlying clusters of the clients, following the setup in previous works \citep{ifca, FCGP} . For ``Balanced Clusters", we first randomly divide the whole dataset into $1000$ clients, and then we group the clients evenly into four clusters as the ground truth. We rotate images in each cluster by {0, 90, 180, 270} degrees, respectively. For ``Imbalanced Clusters", similarly, we rotate clients' images (partitioned into three clusters at a ratio of 2:1:1)  by {0, 90, 180} degrees. We plug  \randname into existing federated clustering methods (IFCA \citep{ifca}, FeSEM \citep{fesem} and FedCAM)  to demonstrate its compatibility. In both settings with balanced or imbalanced clusters, we see that \randname results in more  significantly accurate final models compared with directly privatizing the base clustering algorithms without random rebalancing. We highlight improvement numbers for each setting in green.  The convergence curves under $\varepsilon=2$ and 4 are reported in Figure~\ref{fig:conv_curve}. 

\begin{table*}[h]
\vspace{5pt}
\setlength{\abovecaptionskip}{1em}
\centering
\small
    \resizebox{0.98\linewidth}{!}{
    \renewcommand\arraystretch{1}
        \begin{tabular}{lccc|ccc}
        \thickhline
		& \multicolumn{3}{c}{Balanced Clusters}&\multicolumn{3}{c}{Imbalanced Clusters} \\
      	\multirow{-2}{*}{Methods~~~} & \textcolor{mygreen}{$\varepsilon=2$}  & \textcolor{mygreen}{$\varepsilon=4$} & \textcolor{mygreen}{$\varepsilon=8$}  & \textcolor{mygreen}{$\varepsilon=2$}  & \textcolor{mygreen}{$\varepsilon=4$} & \textcolor{mygreen}{$\varepsilon=8$}   \\
    \hline
    DP-FedAvg & 60.88 & 61.50 & 62.72 & 60.80 & 61.53 & 61.59 \\
    
    \hline
    DP-IFCA & 62.68 & 64.46 & 65.35 & 56.12 & 58.05 & 59.63 \\
    DP-FeSEM & 63.16 & 64.68 & 64.76 & 58.55 & 59.82 & 60.10 \\
    DP-FedCAM & 49.70 & 56.86 & 61.54 & 43.98 & 47.51 & 50.35 \\

    \hline
   \randname (IFCA) & \reshl{65.69}{3.01}&\reshl{66.63}{2.17}&\reshl{67.51}{2.16}& \reshl{61.31}{5.19}&\reshl{61.99}{3.94}&\reshl{63.77}{4.14}\\
    \randname (FeSEM) &\reshl{63.89}{0.73}&\reshl{64.80}{0.12}&\reshl{64.86}{0.10}& \reshl{58.70}{0.15}&\reshl{60.34}{0.52}&\reshl{61.16}{1.06}\\
    \randname (FedCAM) &\reshl{53.91}{4.21}&\reshl{58.65}{1.79}&\reshl{64.78}{3.24}& \reshl{46.12}{2.14}&\reshl{47.95}{0.44}&\reshl{50.78}{0.43}\\
    \thickhline
    \end{tabular}}
    \caption{\small Comparison with baselines on {FashionMNIST}. We validate the effectiveness of our method under various DP budget $\varepsilon$'s. \randname plugged into different clustering methods (the last panel) outperforms the corresponding baselines  that directly privatize the base clustering algorithms without considering uncontrolled cluster sizes.  Full results including comparisons with DP-FedPer are reported in Appendix~\ref{app:tab_with_per}.}
    \label{table:acc_fashionmnist}
\end{table*}


\begin{figure}[h]
    \centering
    \begin{minipage}{0.49\textwidth}
        \centering
        \includegraphics[width=\linewidth]{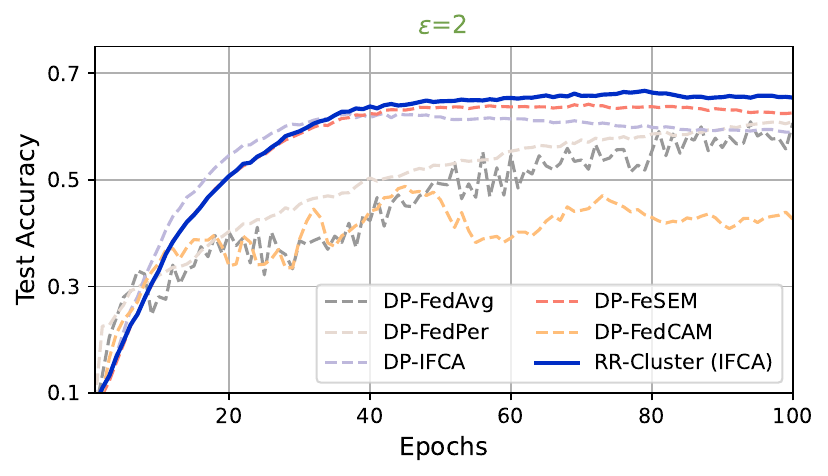}
    \end{minipage}
    \hfill
    \begin{minipage}{0.49\textwidth}
        \centering
        \includegraphics[width=\linewidth]{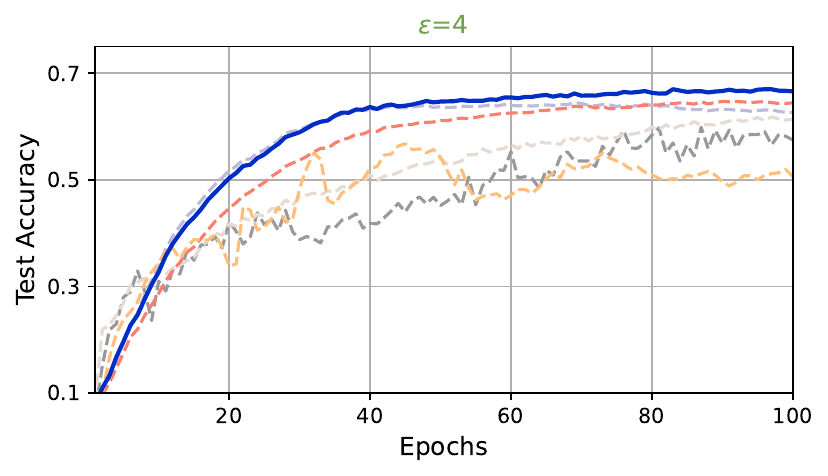}
    \end{minipage}
    \caption{\small Convergence curves compared with strong baselines on {FashionMNIST}. We see that \randname achieves the better  accuracies and faster convergence under both $\varepsilon$ values. 
    }
    \label{fig:conv_curve}
\end{figure}

\paragraph{Text Data with Ambiguous Cluster Patterns.} \label{sec:shakespeare_dataset}
In Table~\ref{table:acc_shakespeare}, we study the performance on the Shakespeare dataset \citep{leaf}, where each speaking role is a client. We do not control or manipulate the underlying clustering patterns a priori. 

\begin{wraptable}{r}{0.48\textwidth}
    \resizebox{0.47\textwidth}{!}{
		\renewcommand\arraystretch{1}
        \begin{tabular}{l|cccc}
		\thickhline
   {Methods~~~} & \textcolor{darkgreen}{$\varepsilon=4$}  & \textcolor{darkgreen}{$\varepsilon=8$} & \textcolor{darkgreen}{$\varepsilon=16$}  \\
    \thickhline
    DP-FedAvg & 04.47 & 13.43 & 17.53 \\
    \hline
    DP-IFCA & 12.62 & 12.64 & 13.20 \\
    DP-FeSEM & 10.97 & 12.99 & 15.89 \\
    DP-FedCAM & - & - & 04.47 \\
    \hline
    \texttt{RR-Cluster (IFCA)} &{13.42}&{13.83}&{16.25}\\
    \hline
    \end{tabular}}
    \vspace{0.05in}
\caption{\small Comparison with baselines on {Shakespeare}. Our proposed approach outperforms the baselines in terms of test accuracies.}
\vspace{1em}
\label{table:acc_shakespeare}
\end{wraptable}
Our method outperforms DP-FedAvg, DP-IFCA, and other clustering baselines that directly integrate DP on this task as well. Note that `$-$' indicates DP-FedCAM  method cannot converge in our experiments. This is due to that it requires clients to send both global model updates and cluster model updates, thus resulting in larger privacy noise. 
Meanwhile, the DP noise can be  significantly reduced by incorporating our \randname into  FedCAM  (Table~\ref{table:acc_fashionmnist}).

\paragraph{Varying Heterogeneity.} \label{sec:emnist_text}
Here, we create datasets with different degrees of heterogeneity. We partition the EMNIST data following a Dirichlet distribution with parameters $\alpha=\{0.5, 0.1\}$ for ``Mild / High Heterogeneity" scenarios respectively. Results are reported in Table~\ref{table:acc_emnist}. We see that (1) most of the clustering methods outperform methods of training a global model in  more heterogeneous setting, and (2) under different privacy $\varepsilon$ values, our method achieves highest accuracies in both heterogeneity settings.

\begin{table*}[h]
\vspace{5pt}
\setlength{\abovecaptionskip}{0.2cm}
\centering
\small
    \resizebox{0.98\linewidth}{!}{
		\renewcommand\arraystretch{1}
        \begin{tabular}{l|ccc|ccc}
		\thickhline
		& \multicolumn{3}{c|}{Mild Client Heterogeneity}&\multicolumn{3}{c}{High Client Heterogeneity} \\
        \multirow{-2}{*}{Methods~~~} & \textcolor{darkgreen}{$\varepsilon=2$}  & \textcolor{darkgreen}{$\varepsilon=4$}  & \textcolor{darkgreen}{$\varepsilon=8$}  & \textcolor{darkgreen}{$\varepsilon=2$}  & \textcolor{darkgreen}{$\varepsilon=4$}  & \textcolor{darkgreen}{$\varepsilon=8$}      \\
    \thickhline
    DP-FedAvg & 26.47 & 30.95 & 32.61 & 23.82 & 24.46 & 25.54 \\
    \hline
    DP-IFCA & 62.23 & 65.39 & 65.93 & 32.53 & 35.60 & 36.19 \\
    DP-FeSEM & 60.43 & 65.72 & 66.24 & 31.12 & 37.52 & 38.74 \\
    DP-FedCAM & 57.35 & 57.22 & 61.26 & 23.66 & 24.55 & 29.70 \\
    \hline
    \randname (IFCA) & \reshl{66.38}{4.15} & \reshl{66.75}{1.03} & \reshl{67.04}{0.80} & \reshl{37.13}{4.60} & \reshl{38.92}{1.40} & \reshl{39.05}{0.31} \\
    \hline
    \end{tabular}}
    \caption{\small Comparison with baselines on the {EMNIST} dataset. We see that \randname outperforms other approaches especially when the privacy parameter is small (i.e., stronger privacy guarantees).}
    \label{table:acc_emnist}
\end{table*}

\subsection{Diagnostic Analysis} 
\label{sec:diagnostic}\label{sec:exp_tradeoff}

\begin{wrapfigure}{r}{0.49\textwidth}
\centering
    \setlength{\abovecaptionskip}{0cm}
    \includegraphics[width=0.48\textwidth]{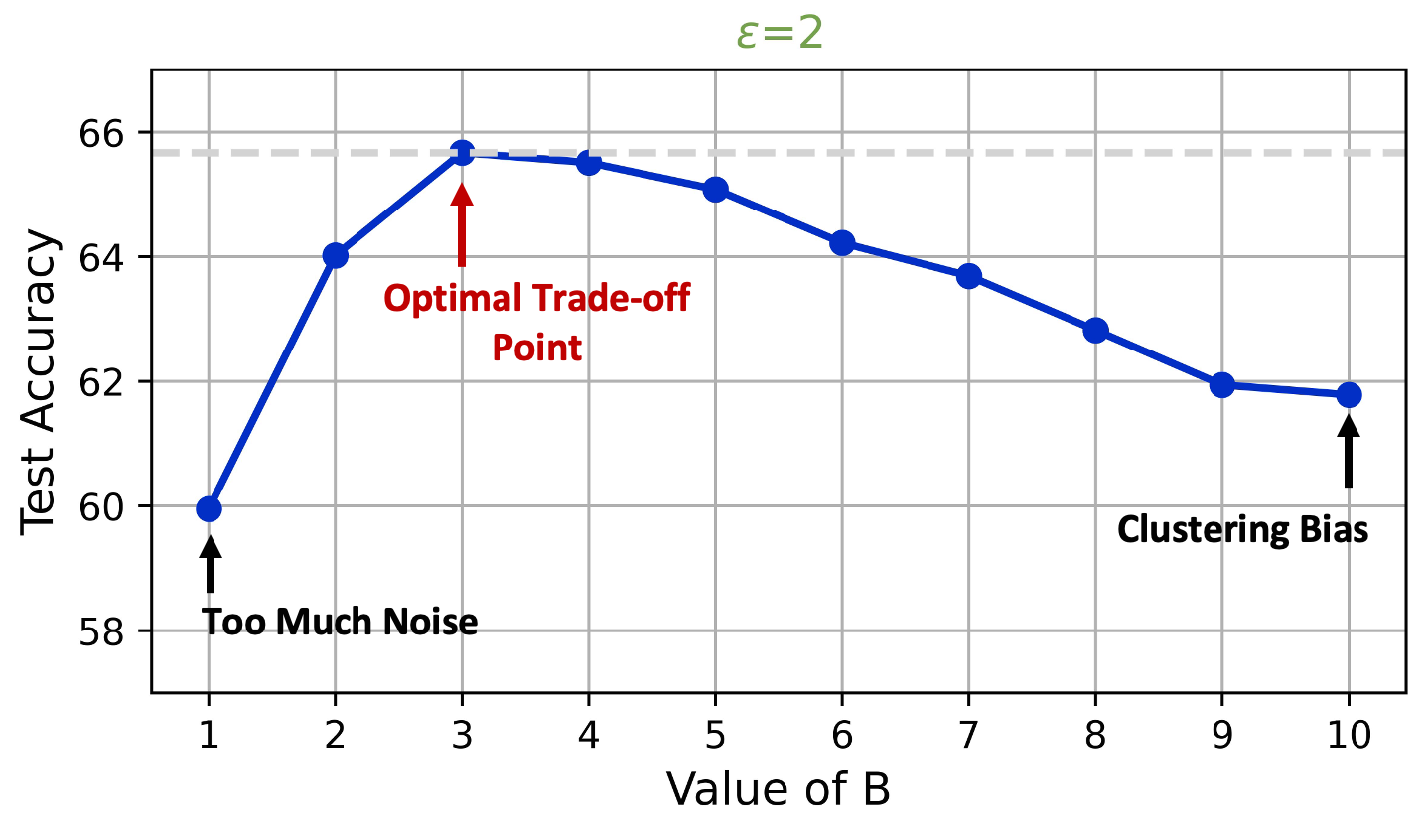}
    \caption{\small Ablation studies on the hyperparameter $B$ on the FashionMNIST dataset. There are a range of $B$'s that lead to improved performance. (Note that the baseline algorithm DP-IFCA in this setting achieves 62.68\% accuracy.) 
    }
    \vspace{0.1in}
    \label{fig:tradeoff_main}
\end{wrapfigure}

\paragraph{Hyperparameter Analysis.} As discussed before, the rebalancing hyperparameter $B$ in our method introduces a tradeoff between privacy noise variance and clustering bias. When $B$ gets larger, more model updates can be sampled from large groups into small ones, leading to potentially larger clustering bias. On the other hand, when $B$ gets smaller, the privacy noise for small clusters may be large as the contributions of a single client in those groups are more difficult to hide. Choosing a proper value of $B$  $(0<B<qM/k)$ is critical for our algorithm. 
We run experiments on the FashionMNIST dataset, with the same image rotation settings in Section~\ref{sec:fashionmnist_text}.
We demonstrate the effects of $B$  in Figure~\ref{fig:tradeoff_main}. By changing the value of $B$, we can find an optimal value that leads to the best tradeoff. We find our method outperforms other private federated clustering baselines for a wide range of $B$'s. 

\paragraph{Clustering Accuracy.} \label{sec:cluster_acc_text}
In addition to the performance of downstream tasks (i.e, learnt model accuracies) reported in previous sections, we also  compare different methods in terms of clustering accuracies---the portion of clients that are correctly clustered.
We experiment on FashionMNIST with balanced rotation as we have access to  the underlying clustering structures. 
We present the success rate in Table~\ref{table:acc_cluster}. Our method achieves the highest clustering accuracy with the increase of $B$. 
Note that our method can exhibit higher clustering accuracies than the baselines even in non-private cases. This implies the seemingly incorrect assignment caused by expanding small clusters may actually help to stabilize clustering, mitigating  model collapse, which often occurs in clustering \citep{fedcam} (also discussed Section~\ref{sec:method:tradeoff}). In the next paragraph, we illustrate this side effect on carefully-generated synthetic data in non-private cases.

\begin{table*}[h]
\vspace{5pt}
\small
\setlength{\abovecaptionskip}{0.2em}
\setlength{\tabcolsep}{11pt}
\centering
		\renewcommand\arraystretch{1}
        \begin{tabular}{l|cccccccc}
		\thickhline
      	\multirow{-1}{*}{Methods~~~} & \textcolor{darkgreen}{$\varepsilon=0.5$}  & \textcolor{darkgreen}{$\varepsilon=2$} & \textcolor{darkgreen}{$\varepsilon=4$}  & \textcolor{darkgreen}{$\varepsilon=8$} & \textcolor{darkgreen}{$\varepsilon=16$}  & \textcolor{darkred}{$\varepsilon\ \rightarrow \infty$} \\
    \hline
    DP-IFCA & 34.37 & 42.18 & 39.06 & 40.62 & 42.18 & 75.00 \\
    DP-FeSEM & - & 42.18 & 39.06 & 32.81 & 50.00 & 46.87 \\
    \hline
    \texttt{RR-Cluster(IFCA)}~B=4 &\text{40.62}&\text{45.31}&\text{43.75}&\text{59.37}&\text{87.50}&\text{98.44}\\
    \texttt{RR-Cluster(IFCA)}~B=6 &\text{40.62}&\text{53.12}&\text{62.50}&\text{87.50}&\text{100.00}&\text{84.37}\\
    \texttt{RR-Cluster(IFCA)}~B=8 &\textbf{40.62}&\textbf{59.37}&\textbf{87.50}&\textbf{98.44}&\textbf{100.00}&\textbf{100.00}\\
    \thickhline
    \end{tabular}
    \caption{\small Accuracy of clustering itself (i.e., percentage of correctly clustered clients) on {FashionMNIST}.}
    \label{table:acc_cluster}
\end{table*}

\paragraph{Case Study on Synthetic Data.}  
To demonstrate \randname's side effect of potentially mitigating model collapse, we conduct experiments on a synthetic dataset. The intuition is that by constraining the number of model updates that are assigned to any cluster is greater than $B$, we can ensure that each cluster model is at least trained using some related data, as opposed to never getting updated. In Appendix~\ref{app:side_effect}, we show that in scenarios where baselines suffer from model collapse, \randname can successfully learn all $k$ clusters.


\section{Conclusion} \label{sec:conclusion}
In this paper, we introduce \randname, a simple add-on that achieves better privacy/utility tradeoffs for federated clustering via a random sampling mechanism that controls client contributions across clusters, thereby reducing the required DP noise. We demonstrate its effectiveness both theoretically and empirically. Further research can explore extending this framework to other clustering problems. 




\bibliographystyle{plainnat}
\bibliography{references}

\begin{thebibliography}{36}
\providecommand{\natexlab}[1]{#1}
\providecommand{\url}[1]{\texttt{#1}}
\expandafter\ifx\csname urlstyle\endcsname\relax
  \providecommand{\doi}[1]{doi: #1}\else
  \providecommand{\doi}{doi: \begingroup \urlstyle{rm}\Url}\fi

\bibitem[TFF()]{TFF}
Tensorflow federated: Machine learning on decentralized data.
\newblock URL \url{https://www.tensorflow.org/federated}.

\bibitem[Arivazhagan et~al.(2019)Arivazhagan, Aggarwal, Singh, and Choudhary]{fedper}
Manoj~Ghuhan Arivazhagan, Vinay Aggarwal, Aaditya~Kumar Singh, and Sunav Choudhary.
\newblock Federated learning with personalization layers.
\newblock \emph{arXiv preprint arXiv:1912.00818}, 2019.

\bibitem[Augello et~al.(2023)Augello, Falzone, and Re]{percom_dcfl}
Andrea Augello, Giulio Falzone, and Giuseppe~Lo Re.
\newblock Dcfl: Dynamic clustered federated learning under differential privacy settings.
\newblock In \emph{PerCom Workshops}, 2023.

\bibitem[Balcan et~al.(2017)Balcan, Dick, Liang, Mou, and Zhang]{balcan2017differentially}
Maria-Florina Balcan, Travis Dick, Yingyu Liang, Wenlong Mou, and Hongyang Zhang.
\newblock Differentially private clustering in high-dimensional euclidean spaces.
\newblock In \emph{International Conference on Machine Learning}, pages 322--331. PMLR, 2017.

\bibitem[Caldas et~al.(2018)Caldas, Duddu, Wu, Li, Kone{\v{c}}n{\`y}, McMahan, Smith, and Talwalkar]{leaf}
Sebastian Caldas, Sai Meher~Karthik Duddu, Peter Wu, Tian Li, Jakub Kone{\v{c}}n{\`y}, H~Brendan McMahan, Virginia Smith, and Ameet Talwalkar.
\newblock Leaf: A benchmark for federated settings.
\newblock \emph{arXiv preprint arXiv:1812.01097}, 2018.

\bibitem[Chang et~al.(2021)Chang, Ghazi, Kumar, and Manurangsi]{chang2021locally}
Alisa Chang, Badih Ghazi, Ravi Kumar, and Pasin Manurangsi.
\newblock Locally private k-means in one round.
\newblock In \emph{International conference on machine learning}, pages 1441--1451. PMLR, 2021.

\bibitem[Cohen et~al.(2017)Cohen, Afshar, Tapson, and Van~Schaik]{emnist}
Gregory Cohen, Saeed Afshar, Jonathan Tapson, and Andre Van~Schaik.
\newblock Emnist: Extending mnist to handwritten letters.
\newblock In \emph{IJCNN}, 2017.

\bibitem[Cohen-Addad et~al.(2022)Cohen-Addad, Epasto, Lattanzi, Mirrokni, Munoz~Medina, Saulpic, Schwiegelshohn, and Vassilvitskii]{cohen2022scalable}
Vincent Cohen-Addad, Alessandro Epasto, Silvio Lattanzi, Vahab Mirrokni, Andres Munoz~Medina, David Saulpic, Chris Schwiegelshohn, and Sergei Vassilvitskii.
\newblock Scalable differentially private clustering via hierarchically separated trees.
\newblock In \emph{Proceedings of the 28th ACM SIGKDD Conference on Knowledge Discovery and Data Mining}, pages 221--230, 2022.

\bibitem[Dennis et~al.(2021)Dennis, Li, and Smith]{dennis2021heterogeneity}
Don~Kurian Dennis, Tian Li, and Virginia Smith.
\newblock Heterogeneity for the win: One-shot federated clustering.
\newblock In \emph{International Conference on Machine Learning}, pages 2611--2620. PMLR, 2021.

\bibitem[Dwork et~al.(2006)Dwork, McSherry, Nissim, and Smith]{dp}
Cynthia Dwork, Frank McSherry, Kobbi Nissim, and Adam Smith.
\newblock Calibrating noise to sensitivity in private data analysis.
\newblock In \emph{Theory of Cryptography: Third Theory of Cryptography Conference, TCC 2006, New York, NY, USA, March 4-7, 2006. Proceedings 3}, pages 265--284. Springer, 2006.

\bibitem[Ghosh et~al.(2020)Ghosh, Chung, Yin, and Ramchandran]{ifca}
Avishek Ghosh, Jichan Chung, Dong Yin, and Kannan Ramchandran.
\newblock An efficient framework for clustered federated learning.
\newblock In \emph{NeurIPS}, 2020.

\bibitem[He et~al.(2023)He, Wang, and Cai]{acsfl}
Zaobo He, Lintao Wang, and Zhipeng Cai.
\newblock Clustered federated learning with adaptive local differential privacy on heterogeneous iot data.
\newblock \emph{IEEE IoTJ}, 2023.

\bibitem[Huang and Liu(2018)]{optim_dp_cluster}
Zhiyi Huang and Jinyan Liu.
\newblock Optimal differentially private algorithms for k-means clustering.
\newblock In \emph{PODS}, 2018.

\bibitem[Kairouz et~al.(2021)Kairouz, McMahan, Avent, Bellet, Bennis, Bhagoji, Bonawitz, Charles, Cormode, Cummings, et~al.]{kairouz2021advances}
Peter Kairouz, H~Brendan McMahan, Brendan Avent, Aur{\'e}lien Bellet, Mehdi Bennis, Arjun~Nitin Bhagoji, Kallista Bonawitz, Zachary Charles, Graham Cormode, Rachel Cummings, et~al.
\newblock Advances and open problems in federated learning.
\newblock \emph{Foundations and trends{\textregistered} in machine learning}, 14\penalty0 (1--2):\penalty0 1--210, 2021.

\bibitem[Kim et~al.(2024)Kim, Kim, and De~Veciana]{FCGP}
Heasung Kim, Hyeji Kim, and Gustavo De~Veciana.
\newblock Clustered federated learning via gradient-based partitioning.
\newblock In \emph{ICML}, 2024.

\bibitem[Krizhevsky et~al.(2009)Krizhevsky, Hinton, et~al.]{cifar}
Alex Krizhevsky, Geoffrey Hinton, et~al.
\newblock Learning multiple layers of features from tiny images.
\newblock 2009.

\bibitem[Li et~al.(2020{\natexlab{a}})Li, Sahu, Talwalkar, and Smith]{tian_fed_survey}
Tian Li, Anit~Kumar Sahu, Ameet Talwalkar, and Virginia Smith.
\newblock Federated learning: Challenges, methods, and future directions.
\newblock \emph{IEEE SPM}, 2020{\natexlab{a}}.

\bibitem[Li et~al.(2022)Li, Zaheer, Reddi, and Smith]{adadps}
Tian Li, Manzil Zaheer, Sashank Reddi, and Virginia Smith.
\newblock Private adaptive optimization with side information.
\newblock In \emph{ICML}, pages 13086--13105, 2022.

\bibitem[Li et~al.(2020{\natexlab{b}})Li, Huang, Yang, Wang, and Zhang]{fedavg_convergence}
Xiang Li, Kaixuan Huang, Wenhao Yang, Shusen Wang, and Zhihua Zhang.
\newblock On the convergence of fedavg on non-iid data.
\newblock In \emph{ICLR}, 2020{\natexlab{b}}.

\bibitem[Li et~al.(2023)Li, Wang, Chi, and Quek]{iotj_dpfedc}
Yiwei Li, Shuai Wang, Chong-Yung Chi, and Tony~QS Quek.
\newblock Differentially private federated clustering over non-iid data.
\newblock \emph{IEEE IoTJ}, 2023.

\bibitem[Liu et~al.(2022)Liu, Hu, Wu, and Smith]{liu2022privacy}
Ken Liu, Shengyuan Hu, Steven~Z Wu, and Virginia Smith.
\newblock On privacy and personalization in cross-silo federated learning.
\newblock \emph{Advances in neural information processing systems}, 35:\penalty0 5925--5940, 2022.

\bibitem[Long et~al.(2023)Long, Xie, Shen, Zhou, Wang, and Jiang]{fesem}
Guodong Long, Ming Xie, Tao Shen, Tianyi Zhou, Xianzhi Wang, and Jing Jiang.
\newblock Multi-center federated learning: clients clustering for better personalization.
\newblock \emph{WWW}, 2023.

\bibitem[Luo et~al.(2024)Luo, Chen, He, Jin, Zhang, and Li]{ppcfl}
Guixun Luo, Naiyue Chen, Jiahuan He, Bingwei Jin, Zhiyuan Zhang, and Yidong Li.
\newblock Privacy-preserving clustering federated learning for non-iid data.
\newblock \emph{FGCS}, 2024.

\bibitem[Ma et~al.(2023)Ma, Zhou, Long, Jiang, and Zhang]{fedcam}
Jie Ma, Tianyi Zhou, Guodong Long, Jing Jiang, and Chengqi Zhang.
\newblock Structured federated learning through clustered additive modeling.
\newblock \emph{NeurIPS}, 2023.

\bibitem[Malekmohammadi et~al.(2025)Malekmohammadi, Taik, and Farnadi]{malekmohammadi2025}
Saber Malekmohammadi, Afaf Taik, and Golnoosh Farnadi.
\newblock Differentially private clustered federated learning.
\newblock \emph{arXiv preprint arXiv:2405.19272}, 2025.

\bibitem[McMahan et~al.(2017{\natexlab{a}})McMahan, Moore, Ramage, Hampson, and y~Arcas]{fedavg}
Brendan McMahan, Eider Moore, Daniel Ramage, Seth Hampson, and Blaise~Aguera y~Arcas.
\newblock Communication-efficient learning of deep networks from decentralized data.
\newblock In \emph{AISTATS}, 2017{\natexlab{a}}.

\bibitem[McMahan et~al.(2017{\natexlab{b}})McMahan, Moore, Ramage, Hampson, and y~Arcas]{shakes1}
Brendan McMahan, Eider Moore, Daniel Ramage, Seth Hampson, and Blaise~Aguera y~Arcas.
\newblock Communication-efficient learning of deep networks from decentralized data.
\newblock In \emph{AISTATS}, 2017{\natexlab{b}}.

\bibitem[McMahan et~al.(2018)McMahan, Ramage, Talwar, and Zhang]{iclr18_dpfedavg}
H~Brendan McMahan, Daniel Ramage, Kunal Talwar, and Li~Zhang.
\newblock Learning differentially private recurrent language models.
\newblock In \emph{ICLR}, 2018.

\bibitem[Mironov(2017)]{renyi_dp}
Ilya Mironov.
\newblock R{\'e}nyi differential privacy.
\newblock In \emph{CSF}. IEEE, 2017.

\bibitem[Nissim and Stemmer(2018)]{dp_cluster}
Kobbi Nissim and Uri Stemmer.
\newblock Clustering algorithms for the centralized and local models.
\newblock In \emph{Algorithmic Learning Theory}, 2018.

\bibitem[Shakespeare(2014)]{shakes2}
William Shakespeare.
\newblock \emph{The complete works of William Shakespeare}.
\newblock Race Point Publishing, 2014.

\bibitem[Shi et~al.(2023)Shi, Liu, Wei, Shen, Wang, and Tao]{dpfedsam}
Yifan Shi, Yingqi Liu, Kang Wei, Li~Shen, Xueqian Wang, and Dacheng Tao.
\newblock Make landscape flatter in differentially private federated learning.
\newblock In \emph{CVPR}, pages 24552--24562, 2023.

\bibitem[Stemmer(2021)]{dp_cluster2}
Uri Stemmer.
\newblock Locally private k-means clustering.
\newblock \emph{JMLR}, 2021.

\bibitem[Wang et~al.(2019)Wang, Balle, and Kasiviswanathan]{subsample_rdp}
Yu-Xiang Wang, Borja Balle, and Shiva~Prasad Kasiviswanathan.
\newblock Subsampled r{\'e}nyi differential privacy and analytical moments accountant.
\newblock In \emph{AISTATS}, pages 1226--1235. PMLR, 2019.

\bibitem[Wu et~al.(2022)Wu, Li, Charles, Xiao, Liu, Xu, and Smith]{motley}
Shanshan Wu, Tian Li, Zachary Charles, Yu~Xiao, Ziyu Liu, Zheng Xu, and Virginia Smith.
\newblock Motley: Benchmarking heterogeneity and personalization in federated learning.
\newblock \emph{arXiv preprint arXiv:2206.09262}, 2022.

\bibitem[Xiao et~al.(2017)Xiao, Rasul, and Vollgraf]{fashionmnist}
Han Xiao, Kashif Rasul, and Roland Vollgraf.
\newblock Fashion-mnist: a novel image dataset for benchmarking machine learning algorithms.
\newblock \emph{arXiv preprint arXiv:1708.07747}, 2017.

\end{thebibliography}


\newpage
\appendix
\onecolumn

\section{Privacy Analysis}

\subsection{RDP to DP Conversion}
\label{app:rdp_to_dp}
We restate existing RDP-to-DP conversion theorem, as follows.
 \begin{theorem}
    (From RDP to $(\epsilon,\delta)$-DP~\citep{renyi_dp}) 
    If mechanism $\mathcal{M}$ satisfies $(\lambda,\epsilon)$-RDP, then it also satisfies $(\epsilon+\frac{\log(1/\delta)}{\lambda-1},\delta)$-DP for any $\delta \in (0,1)$.
\label{thm:rdp_to_dp}
\end{theorem}

\subsection{Sensitivity Analysis}
\label{app:sensitivity}
Continuing from Section~\ref{sec:privacy_analysis}, we study the sensitivity of our clustering updates each iteration for client-level DP. 
In global round $t$, we select $qM$ clients to update $k$ cluster models.
Denote $\Delta{\theta}^t_i$ as the model updates from the $i$-th client at round $t$. By clipping, the $L_2$-norm of each update $\Delta{\theta}^t_{i}$ are limited within $[0,C_\theta]$. 
According to Algorithm~\ref{alg:pseudo-code-general}, we add noise to the sum of cluster updates within each cluster.  In the following, we investigate the $L_2$ sensitivity of the sum of cluster updates for any cluster under the client-level DP setup. 
\begin{proof}
Suppose that we have two adjacent client sets, $\Theta$ and $\Theta'$, where $\Theta'$ has one more/less client than $\Theta^t$. 
We denote the added/removed update as $\Delta{\theta}_A$.
We note that adding/removing a client may or may not affect other groups of model updates based on our algorithm. Without loss of generality, we assume the added/removed client update belongs to cluster $1$. We discuss two cases  as follows.

\textbf{Case~1}: Adding or removing one client belonging to the groups assigned to cluster 1 does not trigger random sampling and rebalancing. For instance, cluster 1 after removing one client still satisfies that the cardinality is at least $B$. For the sensitivity of the function outputting $k$ cluster's model updates, we only need to consider cluster 1, as all other groups remain unchanged. Denote the set of model updates assigned to cluster 1 before removing $\Delta \theta_{A}$ as $\{\Delta\theta_{1,1}, \cdots, \Delta\theta_{1,n_1}, \Delta{\theta}_{A}\}$. Then the sensitivity of summing up model updates in cluster 1 can be upper bounded by
\begin{equation}
\begin{aligned}
\max_{\Theta, \Theta'} ~\left\lVert\left(\Delta\theta_{1,1}+...+\Delta\theta_{1,n_1}+\Delta{\theta}_{A}\right)  - \left(\Delta\theta_{1,1}+...+\Delta\theta_{1,n_1}\right)\right\lVert = \lVert \Delta\theta_{A} \lVert \leq C_\theta.
\end{aligned}
\end{equation}
If we add $\Delta \theta_A$, the same sensitivity holds.

\textbf{Case~2}: The added/removed client triggers random sampling from large clusters to expand small ones. For example, removing a client results in cluster 1 originally having $B$ updates to only have $B-1$ updates, in which case we would need to sample one model update from another cluster (denoted as $\Delta \theta'$) and insert it into cluster 1. 
The sensitivity of summing up model updates in cluster 1 is 
\begin{equation}
\begin{aligned}
&\max_{\Theta, \Theta'}  ~\lVert\left(\Delta\theta_{1,1}+...+\Delta\theta_{1,n_1}+\Delta{\theta}'\right)  - \left(\Delta\theta_{1,1}+...+\Delta\theta_{1,n_1}+\Delta\theta_A\right)\lVert\\
&\leq \lVert \Delta\theta_{A} \lVert + \lVert \Delta\theta' \lVert \leq 2C_\theta.
\end{aligned}
\end{equation}
Similarly, we have that the sensitivity corresponding to the cluster that we sample $\Delta \theta'$ from is $C_{\theta}$. 

Hence, we prove the overall sensitivity is at most $2C_{\theta}$. Combining with the fact that we privatize each cluster model separately and $k$ new cluster model are obtained based on disjoint sets of clients, we can obtain desired privacy bounds.
\end{proof}

\subsection{Privacy Guarantees}
\label{app:priv_guarantee_rdp}
Here we provide the proof of Theorem~\ref{thm:overall_privacy}. We first introduce the definition of R\'enyi differential privacy and several theorems used in our privacy analysis.

\begin{definition} [R\'enyi differential privacy~\citep{renyi_dp}]
    \label{thm:def_rdp}
    A randomized mechanism $f : \mathcal{D} \to \mathcal{R}$ is said to have $\varepsilon$-R\'enyi differential privacy of order $\alpha$, or$( \alpha, \varepsilon )$-RDP for short, if for any adjacent $D, D' \in \mathcal{D}$ it holds that
    \begin{equation}
        D_\alpha(f(D) \| f(D')) = \frac{1}{\alpha - 1} \log \mathbb{E}_{x \sim f(D')} \left( \frac{f(D)}{f(D')} \right)^\alpha \leq \varepsilon.
    \end{equation}
\end{definition}

\begin{theorem} [RDP adaptive composition~\citep{renyi_dp}]
   \label{thm:rdp_composition} 
   Let $f : \mathcal{D} \to \mathcal{R}_1$ be $(\alpha, \varepsilon_1)$-RDP and $g : \mathcal{R}_1 \times \mathcal{D} \to \mathcal{R}_2$ be $(\alpha, \varepsilon_2)$-RDP, then the mechanism defined as $(X, Y)$, where $X \sim f(D)$ and $Y \sim g(X, D)$, satisfies $(\alpha, \varepsilon_1 + \varepsilon_2)$-RDP.
\end{theorem}

\begin{theorem} [RDP with Gaussian mechanism~\citep{renyi_dp}]
    \label{thm:gm_in_rdp}
    If \( f \) has sensitivity \( C \), then the Gaussian mechanism \( G_{\sigma} f(D) = f(D) + N(0, C^2\sigma^2) \) satisfies \( \left( \alpha, \frac{\alpha}{2\sigma^2} \right) \)-RDP.
\end{theorem}

\begin{theorem} [RDP with subsampling~\citep{subsample_rdp}]
    \label{thm:sampling_in_rdp}
    Given a randomized mechanism $\mathcal{M}$, and let the randomized algorithm $\mathcal{M} \circ \text{subsample}$ be defined as: (1) \textit{subsample}: subsample with subsampling rate $q$; (2) \textit{apply} $\mathcal{M}$: a randomized algorithm taking the subsampled dataset as the input. For all integers $\alpha \geq 2$, if $\mathcal{M}$ is $(\alpha, \varepsilon)\text{-RDP}$, then $\mathcal{M} \circ \text{subsample}$ is $(\alpha, \varepsilon_q)\text{-RDP}$ where
    \begin{equation}
    \begin{aligned}
        \varepsilon_q \leq & \frac{1}{\alpha - 1} \log \left( 1 + q^2 \binom{\alpha}{2} \min \left\{ 4(e^{\varepsilon(2)} - 1), e^{\varepsilon(2)} \min \left[ 2, (e^{\varepsilon(\infty)} - 1)^2 \right] \right\} \right) \\
        & + \sum_{j=3}^{\alpha} q^j \binom{\alpha}{j} e^{(j-1)\varepsilon(j)} \min \left\{ 2, (e^{\varepsilon(\infty)} - 1)^j \right\}.
    \end{aligned}
    \end{equation}
\label{thm:subsample}
\end{theorem}
\vspace{-0.5cm}

In Theorem~\ref{thm:subsample}, DP parameter $\varepsilon$ is denoted as $\varepsilon(\alpha)$, since $\varepsilon$ can be viewed as a function of RDP parameter $\alpha$ for $1 \leq \alpha \leq \infty$~\citep{subsample_rdp}.

\begin{theorem} [Converting RDP into $(\varepsilon,\delta)$-DP~\citep{renyi_dp}]
\label{thm:rdp2dp}
    If $f$ is an $(\alpha, \varepsilon)$-RDP mechanism, it also satisfies $(\varepsilon + \frac{\log 1/\delta}{\alpha - 1}, \delta)$-differential privacy for any $0 < \delta < 1$.
\end{theorem}

We now provide the overall privacy guarantee of \randname, stated in Theorem~\ref{thm:rdp_T_rounds}. We analyze the privacy budget by considering three steps: (1) adaptive composition of private identifier and private model parameters in a single round, (2) privacy amplification with subsampling ratio $q$, and (3) composition over $T$ rounds.

We first provide the privacy guarantee for a single round \randname without subsampling. We restate the Proposition~\ref{thm:rdp_single_round} in the main text and provide the proof.


\begin{proof}
    According to Theorem~\ref{thm:gm_in_rdp}, if using noise $\sigma_\theta=\sqrt{\frac{\alpha}{2\varepsilon_2}}$, we can calculate the privacy budget $\varepsilon$ as $\frac{\alpha}{2\sigma_\theta^2} = \varepsilon_2$. Which means for model parameters, it gives $(\alpha, \varepsilon_2)$-RDP. We then use adaptive RDP composition using Theorem~\ref{thm:rdp_composition} to compose budget $\varepsilon_1$ for identifiers and $\varepsilon_2$ for parameters to get $(\alpha, \varepsilon_1+\varepsilon_2)$-RDP in total.
\end{proof}

We then amplify the $(\alpha, \varepsilon_1+\varepsilon_2)$-RDP budget with client subsampling with ratio $q$ and then compose it over $T$ rounds to provide overall privacy guarantee as follows. We denote $\varepsilon$ as $\varepsilon(\alpha)$ as in Theorem~\ref{thm:subsample}.


\rdpTrounds*

\vspace{-0.5cm}
\setlength{\abovedisplayskip}{2pt}
\begin{align*}
\begin{split}
    \varepsilon(\alpha) &\leq \frac{T}{\alpha - 1} \cdot \log \biggl( 1 + q^2 \binom{\alpha}{2} \min \Bigl\{ 4(e^{\varepsilon_1(2) + \varepsilon_2(2)} - 1), \quad e^{\varepsilon_1(2) + \varepsilon_2(2)} \min \{ 2, (e^{\varepsilon_1(\infty)+\varepsilon_2(\infty)} - 1)^2 \} \Bigr\} \\
    &+ \sum_{j=3}^{\alpha} q^j \binom{\alpha}{j} e^{(j-1)(\varepsilon_1(j)+\varepsilon_2(j)}) \min \{ 2,(e^{\varepsilon_1(\infty)+\varepsilon_2(\infty)} - 1)^j \} \biggr).\\
\end{split}
\end{align*}

\begin{proof}
    Given $(\alpha, \varepsilon_1+\varepsilon_2)$-RDP in one round with sampling rate $q$, we can use Theorem~\ref{thm:sampling_in_rdp} as
    \begin{align}
        \varepsilon(\alpha) &\leq \frac{1}{\alpha - 1} \cdot \log \biggl( 1 + q^2 \binom{\alpha}{2} \min \Bigl\{ 4(e^{\varepsilon_1(2) + \varepsilon_2(2)} - 1), \quad e^{\varepsilon_1(2) + \varepsilon_2(2)} \min \{ 2, (e^{\varepsilon_1(\infty)+\varepsilon_2(\infty)} - 1)^2 \} \Bigr\} \nonumber \\
        &+ \sum_{j=3}^{\alpha} q^j \binom{\alpha}{j} e^{(j-1)(\varepsilon_1(j)+\varepsilon_2(j)}) \min \{ 2,(e^{\varepsilon_1(\infty)+\varepsilon_2(\infty)} - 1)^j \} \biggr).
    \end{align}
    Then, we compose $\varepsilon_q$ over $T$ rounds still using adaptive composition as in Theorem~\ref{thm:rdp_composition} as 
    \setlength{\belowdisplayskip}{2pt}
    \begin{align}
        \varepsilon(\alpha) &\leq \frac{T}{\alpha - 1} \cdot \log \biggl( 1 + q^2 \binom{\alpha}{2} \min \Bigl\{ 4(e^{\varepsilon_1(2) + \varepsilon_2(2)} - 1), \quad e^{\varepsilon_1(2) + \varepsilon_2(2)} \min \{ 2, (e^{\varepsilon_1(\infty)+\varepsilon_2(\infty)} - 1)^2 \} \Bigr\} \nonumber \\
        &+ \sum_{j=3}^{\alpha} q^j \binom{\alpha}{j} e^{(j-1)(\varepsilon_1(j)+\varepsilon_2(j)}) \min \{ 2,(e^{\varepsilon_1(\infty)+\varepsilon_2(\infty)} - 1)^j \} \biggr).
    \end{align}
\end{proof}

Finally, we can convert the $(\alpha, \varepsilon (\alpha))$-RDP bound DP using Theorem~\ref{thm:rdp2dp}. We have that for any $\delta \in (0,1)$ (we used $0.001 < 1/M$ in our experiemtns), our algorithm satisfies $(\varepsilon, \delta)$-DP where
\begin{align}
        \varepsilon = & \varepsilon(\alpha) + \frac{\log{1/\delta}}{\alpha-1} \nonumber \\
        \leq & \frac{T}{\alpha - 1} \cdot \log \biggl( 1 + q^2 \binom{\alpha}{2} \min \Bigl\{ 4(e^{\varepsilon_1(2) + \varepsilon_2(2)} - 1), \quad e^{\varepsilon_1(2) + \varepsilon_2(2)} \min \{ 2, (e^{\varepsilon_1(\infty)+\varepsilon_2(\infty)} - 1)^2 \} \Bigr\} \nonumber \\
        &+ \sum_{j=3}^{\alpha} q^j \binom{\alpha}{j} e^{(j-1)(\varepsilon_1(j)+\varepsilon_2(j)}) \min \{ 2,(e^{\varepsilon_1(\infty)+\varepsilon_2(\infty)} - 1)^j \} \biggr).
\end{align}


\section{Convergence Analysis}

\subsection{Proof of Lemma~\ref{lemma:mis_cls}}
\label{app:err_cluster_prob}
\label{app:convergence}
In this section, we present convergence analysis of the proposed method. In order to streamline our analysis, we make several simplifications: we assume that all clients participate in every rounds. In addition, when considering the clipping operation, we assume the clipping threshold is large enough thus did not introducing bias.



\begin{proof}
Suppose we have cluster models $\{\theta_j\}_{j \in [k]}$, and assume that after certain steps, we have $\|\theta_j - \theta_j^*\| \leq (\frac{1}{2}-\beta)\sqrt{\frac{\lambda}{L}}\Delta, j \in [k]$.
We first analyze the probability of incorrectly clustering a client taking into consideration the following three factors.
\begin{itemize}
    \item Factor 1. Clients selecting their clusters by calculating training losses can introduce some inherent clustering error. We denote the probability of such incorrect  clustering as $\Pr(E)$. 
    \item Factor 2. Clients privatizing cluster identifiers introduces additional error, where the probability is denoted as $\Pr(H)$.
    \item Factor 3. \randname resamples updates from large to small groups, which also increases clustering error, denoted as $\Pr(O)$.
\end{itemize}

\textbf{Factor 1}: 
We begin with the definition of $E_i^{j, j'}$, which denotes an event where in a certain global round, client $i$ who belongs to cluster $j$ by ground truth, however been changed its cluster assignment to cluster $j' (j' \neq j)$ , due to the min-loss cluster selection strategy. 
We bound this probability of erroneous cluster selection from Factor 1 as follows. Without loss of generality, we consider $E_i^{1, j}$. We have
\begin{equation}
\Pr\left(E_i^{1, j}\right) \leq \Pr\left(F_i(\theta_1) \geq F_i(\theta_{j})\right)\leq \Pr\left(F_i(\theta_1)>t\right)+\Pr\left(F_i(\theta_{j})\leq t\right)
\end{equation}
for all $t \geq 0$. Choosing $t=\frac{F^1(\theta_1)+F^1(\theta_j)}{2}$, we have
\begin{equation}
\begin{aligned}
\Pr(F_i(\theta_1)>t) =& \Pr\left(F_i(\theta_1) > \frac{F^1(\theta_1)+F^1(\theta_j)}{2}\right) \\
=& \Pr \left(F_i(\theta_1) - F^1(\theta_1) > \frac{F^1(\theta_j)-F^1(\theta_1)}{2}\right)
\end{aligned}
\end{equation}
for the first term. Considering that $F^1$ is $\lambda$-strongly convex and $L$-smooth and the assumption that $\|\theta_j^t-\theta_j^*\|\leq (\frac{1}{2}-\beta) \sqrt{\frac{\lambda}{L} \Delta}$ where $\Delta := \min_{j \neq j'} \|\theta_j^*-\theta_{j'}^*\|$, we have
\begin{equation}
F^1(\theta_j) \geq F^1(\theta_1^*) + \frac{\lambda}{2}\|\theta_j - \theta_1^*\|^2 \geq F^1(\theta_1^*) + \frac{\lambda \Delta^2}{2}(\frac{1}{2}+\beta)^2
\end{equation}
and
\begin{equation}
F^1(\theta_1) \leq F^1(\theta_1^*) + \frac{L}{2}\|\theta_1 - \theta_1^*\|^2 \leq  F^1(\theta_1^*) + \frac{\lambda \Delta^2}{2}(\frac{1}{2}-\beta)^2.
\end{equation}
Thus we have that
\begin{equation}
 \frac{F^1(\theta_j)-F^1(\theta_1)}{2} \geq \frac{\lambda \Delta^2}{2}(\frac{1}{2}+\beta)^2 - \frac{\lambda \Delta^2}{2}(\frac{1}{2}-\beta)^2 = \frac{\beta \lambda \Delta^2}{2}.
\end{equation}
According to the Chebyshev's inequality and the bounded variance assumption, we obtain $\Pr(F_i(\theta_1)>t) \leq \frac{4\eta^2}{\beta^2 \lambda^2 \Delta^4}$ and similarly,  $\Pr(F_i(\theta_j)\leq t) \leq \frac{4\eta^2}{\beta^2 \lambda^2 \Delta^4}$, and thus $\Pr(E_i^{j, j'}) \leq \frac{8\eta^2}{\beta^2 \lambda^2 \Delta^4}$.\\

\textbf{Factor 2}: We analyze the probability that adding random Gaussian noise to the one-hot cluster identifiers changes the index with the maximum value. 
Let $\sigma_s^2$ denote the variance of the Gaussian noise. We consider the event $H_i^{j, {j'}}$, which means client $i$ is assigned to some cluster $j$ ($j \in [k]$) by the min-loss cluster selection strategy, but is changed into cluster $j'$ due to the added noise to guarantee differential privacy. In this step, we add Gaussian noise with variance $\sigma_s^2$ to each coordinate of the one-hot  $s_i$, resulting in $\widetilde{s_i}$. Therefore,
\begin{equation}
    \Pr(H_i^{j, {j'}}) \leq \Pr\left(1+\mathcal{N}(0,\sigma_s^2) \leq 0+\mathcal{N}(0,\sigma_s^2) \right) = \Pr(\mathcal{N}(0,2\sigma_s^2) \leq -1).
\end{equation}
Using the tail bound of Gaussian distributions $\Pr(x\geq t) \leq \frac{\sigma_s}{t\sqrt{2\pi}}\exp(-t^2/2\sigma_s^2)$, we have
\begin{equation}
    \Pr(H_i^{j, {j'}}) \leq \frac{\sigma_s}{\sqrt{\pi}}\exp(-1/4\sigma_s^2).
\end{equation}



\noindent
\textbf{Factor 3}:
We finally analyze the probability of event $O_i^{j, j'}$, where client $i$ is assigned to cluster $j$ after the previous two steps, but the assignment gets altered due to random rebalancing.
Generally, if the mapping between a model update and a cluster has changed, then the model update should be  in a large cluster $S_j, j \in [j_l]$.
Assume that there are $h$ updates in a large cluster that are needed to merge into small clusters, we have
\begin{equation}
    \Pr(O_i^{j, j'}) \leq \Pr(|S_j|>B)\cdot \frac{h}{\sum_{j\in j_l}|S_j|}.
\end{equation}
For large clusters, considering the average cluster scale is $M/k$ by definition, and the variance of cluster scale is bounded by $\mu^2$. With one-sided Chebyshev's inequality, we can have 
\begin{equation}
    \Pr(|S_j| > B) \leq \left(\frac{\mu}{M/k-B}\right)^2.
\end{equation}
Combined with the facts that $h < M$ and $\sum_{j\in j_l}|S_j| > M/k$, it holds that
\begin{equation}
\Pr(O_i^{j, j'}) \leq \left(\frac{\mu}{M/k-B}\right)^2 \cdot \frac{M}{M/k} = \frac{k\mu^2}{(M/k-B)^2}.
\end{equation}

If the final cluster assignment is incorrect, then the error must have occurred in at least one of the three steps. Hence, we can bound the error clustering rate $\tau$ as
\begin{equation}
\begin{aligned}
    \tau_i^{j,j'} \leq \Pr(E_i^{j,j'})+\Pr(H_i^{j,j'})+\Pr(O_i^{j,j'}) 
\leq \frac{8\eta^2}{\beta^2 \lambda^2 \Delta^4} + \frac{\sigma_s}{\sqrt{\pi}}\exp(-1/4\sigma_s^2) + \frac{k\mu^2}{(M/k-B)^2}.
\end{aligned}
\end{equation}

Note that $\tau_i^{j,j'}$ denotes the probability of client update $i$ being clustered into any single cluster $j'$. We can get the probability of any client update $i$ been wrongly clustered into any other cluster as 
\begin{equation}
    \tau = \bigcup_{j'\in[k]}\{\tau_i^{j,j'}\} \leq k\cdot \tau_i^{j,j'} 
     \leq \frac{8\eta^2k}{\beta^2 \lambda^2 \Delta^4} + \frac{\sigma_sk}{\sqrt{\pi}}\exp(-1/4\sigma_s^2) + \frac{k^2\mu^2}{(M/k-B)^2}.
\end{equation}
\end{proof}

\subsection{Proof of Theorem~\ref{thm:one_step}}
\label{app:conv}

\begin{proof}
Suppose that at a certain communication round, we have that $\|\theta_j-\theta_j^*\| \leq (\frac{1}{2}-\beta)\sqrt{\frac{\lambda}{L}}\Delta$, for all $j \in [k]$. With out loss of generality, we focus on the update of cluster $1$. Based on the updating rule, we have 
\begin{equation}
\left\|\theta^+_1-\theta^*_1\right\|= \left\|\theta_1-\theta^*_1-\frac{\gamma}{|S_1|} \left(\sum_{i\in S_1}\nabla F_i(\theta_1)+\mathcal{N}(0,(2C_\theta)^2\sigma_\theta^2I_{d})\right) \right\|
\end{equation}
where $F_i(\theta) = F(\theta, D_i)$ with $D_i$ being the set of data from the $i$-th client used to compute the gradient. 
For cluster $j$, we denote its ground-truth client set as $S_j^*$, and the complement set as $\overline{S_j^*}$. Since  $ S_1 = (S_1\cap S_1^*)\cup(S_1\cap\overline{S_1^*})\ $,
we have
\begin{equation}
\|\theta^+_1-\theta^*_1\|
=\|\underbrace{\theta_1-\theta^*_1-\frac{\gamma}{|S_1|}\sum_{i\in S_1\cap S^*}\nabla F_i(\theta_1)}_{T_1}
-\underbrace{\frac{\gamma}{|S_1|}\sum_{i\in S_1\cap \overline{S_i^*}}\nabla F_i(\theta_1)}_{T2}
-\underbrace{\frac{\gamma}{|S_1|}\mathcal{N}(0,(2\sigma_\theta C_\theta)^2 I_{d})}_{T3} \|. 
\end{equation}
Using triangle inequality, we obtain
$\|\theta^+_1-\theta^*_1\|\leq\|T_1\|+\|T_2\|+\|T_3\|$.

\noindent
\textbf{Bound $\|T_1\|$}: We split $T_1$ as follows:
\begin{equation}
T_1=\underbrace{\theta_1-\theta_1^*-\widehat{\gamma} \nabla F^1\left(\theta_1\right)}_{T_{11}}+\widehat{\gamma}(\underbrace{\nabla F^1\left(\theta_1\right)- \frac{1}{|S_1\cap S^*_1|}\sum_{i \in S_1 \cap S_1^*} \nabla F_i\left(\theta_1\right)}_{T_{12}}),
\end{equation}
where $\widehat{\gamma}:=\frac{\gamma|S_1\cap S^*_1|}{|S_1|}$. Based on $\lambda$-strongly convexity and $L$-smoothness of $F^1$, we can bound $\|T_{11}\|$ as 
\begin{equation}
\|T_{11}\| = \|\theta_1-\theta_1^*-\widehat{\gamma} \nabla F^1\left(\theta_1\right)\| \leq \left(1-\frac{\widehat{\gamma}\lambda L}{\lambda + L}\right) \|\theta_1 - \theta_1^*\|
\end{equation}
when $\widehat{\gamma}<\frac{1}{L}$.
For $T_{12}$, we have $\mathbb{E}[\|T_{12}\|^2]= \frac{v^2}{|S_1\cap S_1^*|}$, which means $E[\|T_{12}\|] \leq \frac{v}{\sqrt{|S^1 \cap S_1^*|}}$. We can then bound the $\|T_{12}\|$ by Markov's inequality as: for any $\delta_1 \in (0, 1]$, we have with probability at least $1-\delta_1$,
\begin{equation}
\|T_{12}\|\leq\frac{v}{\delta_1\sqrt{|S_1 \cap S^*_1|}}.
\end{equation}
Taking into accountant $T_{11}$, $T_{12}$ and fact that $\lambda < L$, we have with probability at least $1-\delta_1$,
\begin{equation}
\|T_1\| \leq \left(1-\frac{\gamma L|S_1\cap S_1^*|}{2|S_1|}\right)\|\theta_1 - \theta^*_1\| + \frac{v}{\delta_1 \sqrt{|S_1 \cap S_1^*|}}.
\end{equation}

\noindent
\textbf{Bound $\|T_2\|$}: We propose to split the $(S_1 \cap \overline{S_1^*})$ into $\bigcup_{j\neq1, j\in[k]}(S_1\cap S^*_j)$. Without loss of generality, we first analyze the $(S_1 \cap S_j^*)$ as \\
\begin{equation}
T_{2j} = |S_1\cap S^*_j|\underbrace{\nabla F^j(\theta_1)}_{T_{21j}}+\underbrace{\sum_{i\in S_i \cap S_j^*}(\nabla F_i(\theta_1) - \nabla F^j(\theta_1))}_{T_{22j}}, 
\end{equation}
where $T_2 = \frac{\gamma}{|S_1|}\sum_{j\in [k]}T_{2j}$
About $T_{21j}$, we have
\begin{equation}
\begin{aligned}
\|T_{21j}\| &= \|\nabla F^j(\theta_1) - \nabla F^j(\theta_j^*)\| \leq L\|\theta_1-\theta_j^*\| \\
&\leq L\|\theta_1-\theta_1^*\| + L\|\theta_1^*-\theta_j^*\| \leq \frac{3}{2}L\Delta.
\end{aligned}
\end{equation}
Further, we have
$
\mathbb{E}\left[\left\|T_{22j}\right\|^2\right]=\left|S_1 \cap S_j^*\right| {v^2}, 
$
which implies 
$
\mathbb{E}\left[\left\|T_{22j}\right\|\right]\leq \sqrt{\left|S_1 \cap S_j^*\right|} {v}.
$
And according to Markov's inequality, for any $\delta_2\in[0, 1]$, with probability at least $1-\delta_2$, we have
\begin{equation}
\|T_{22j}\| = \left\|\sum_{i \in S_1 \cap S_j^*} \nabla F_i\left(\theta_1\right)-\nabla F^j\left(\theta_1\right)\right\|\leq \frac{\sqrt{\left|S_1 \cap S_j^*\right|} {v}}{\delta_2}.
\end{equation}
Considering there are $k$ clusters, we have with probability at least $(1-k\delta_2)$,  
\begin{equation}
\|T_2\|\leq \frac{3\gamma L\Delta}{2|S_1|}|S_1\cap \overline{S_1^*}|+\frac{\gamma v\sqrt{k}}{\delta_2|S_1|}\sqrt{|S_1\cap \overline{S_1^*}|}.
\end{equation}

\textbf{Bound $\|T_3\|$}: We give the norm bound of the gaussian noise from DP as follows.
$\|T_3\| = \|\frac{\gamma}{|S_1|}\mathcal{N}(0,(2\sigma_\theta C_{\theta})^2I_{d})\|$, which is a vector with dimension of model parameters $d$ and scale of DP noise multiplier. First, we have
\begin{equation}
\begin{aligned}
\|T_3\| &= \bigl\|\frac{\gamma}{|S_1|}\mathcal{N}(0,(2\sigma_\theta C_\theta)^2 I_{d})\bigr\| \leq \frac{2\gamma\sigma_\theta C_\theta}{B} \,\|\mathcal{N}(0, I_{d})\|.
\end{aligned}
\end{equation}

Denote $X \sim \mathcal{N}(0, I_{d})$. Then $\|X\|^2$ follows a $\chi^2$-distribution with $d$ degrees of freedom, i.e. $\|X\|^2 \sim \chi^2_{d}$. 
We then use the upper bound for a \(\chi^2\)-tail to bound $\|T_3\|$ as
\begin{equation}
   P \left(\|X\|^2 \leq d 
                 + 2\sqrt{d\ln(\tfrac{1}{\delta_3})} 
                 + 2\ln(\tfrac{1}{\delta_3})\right)
   \geq 1-\delta_3.
\end{equation}
Multiplying both sides by $(\sigma_\theta/B)$ and taking the square root yields
\begin{equation}
   P\left(\|T_3\| \leq \frac{2\gamma\sigma_\theta C_\theta}{B} 
         \sqrt{d 
                 + 2\sqrt{d\ln(\tfrac{1}{\delta_3})} 
                 + 2\ln(\tfrac{1}{\delta_3})}\right)
   \geq 1-\delta_3.
\end{equation}
Which means that we have probability at least $1-\delta_3$ that 
\begin{equation}
    \|T_3\| \leq \frac{2\gamma\sigma_\theta C_\theta}{B} 
         \sqrt{d 
                 + 2\sqrt{d\ln(\tfrac{1}{\delta_3})} 
                 + 2\ln(\tfrac{1}{\delta_3})}.
\end{equation}

With bounded $\|T_1\|$ and $\|T_2\|$ above, we have probability at least $1-(\delta_1+k\delta_2+\delta_3)$ 
that
\begin{equation}
\begin{aligned}
\|\theta^+_1-\theta^*_1\|\leq& 
\left(1-\frac{\gamma L|S_1\cap S_1^*|}{2|S_1|}\right)\|\theta_1 - \theta^*_1\| \\
+& \frac{v}{\delta_1 \sqrt{|S_1 \cap S_1^*|}}
+  \frac{3\gamma L\Delta}{2|S_1|}|S_1\cap \overline{S_1^*}| \\
+& \frac{\gamma v\sqrt{k}}{\delta_2|S_1|}\sqrt{|S_1\cap \overline{S_1^*}|} 
+ \frac{2\gamma\sigma_\theta C_\theta}{B}\sqrt{d + 2\sqrt{d\ln(\tfrac{1}{\delta_3})} + 2\ln(\tfrac{1}{\delta_3})}.
\end{aligned}
\end{equation}
We bound $|S_1|$, $|S_1\cap S_1^*|$ and $|S_1\cap \overline{S_1^*}|$ next to complete the theorem. Based on the updating rules of our method, we note that the maximum cardinality of any $S_j$ is bounded by $\rho := M-(k-1)B$, thus we have $|S_1|\leq \rho$. Also, after rebalancing, we have $|S_1|\geq B$.
Given the error rate in clustering from Lemma~\ref{lemma:mis_cls}, we have $ \mathbb{E}[|S_1\cap \overline{S_1^*|}] \leq \tau M $, plus the Markov's inequality, we obtain probability at least $1-\delta_4$ that $ |S_1\cap \overline{S_1^*|}\leq \tau M/\delta_4$, meanwhile, we have $|S_1\cap S_1^*| = |S_1|-(|S_1\cap \overline{S_1^*}|) \geq (B-\frac{\tau M}{\delta_4})$.
Let $\delta_c \geq \delta_1+k\delta_2+\delta_3+\delta_4$ and choose $\delta_1=\delta_c/4, \delta_2=\delta_c/4k, \delta_3 = \delta_c/4, \delta_4 = \delta_c/4 $, then the failure probability can be bounded by $\delta_c$. 
Finally, we have with probability at least $(1-\delta_c)$, it holds that
\begin{equation}
\begin{aligned}
\|\theta^+_1-\theta^*_1\|\leq& \left(1-\frac{\gamma L (B-\frac{2\tau M}{\delta_c})}{2\rho}\right)\|\theta_1 - \theta^*_1\| \\
+&\frac{4 v}{\sqrt{\delta_c(B\delta_c-4\tau M)}}
+\frac{6 \tau \gamma L\Delta M}{\delta_c B}
+\frac{8\gamma v k\sqrt{\tau k M}}{B\delta_c\sqrt{\delta_c}}
+\frac{2\gamma\sigma_\theta C_\theta}{B}\sqrt{d + 2\sqrt{d\ln(\tfrac{4}{\delta_c})} + 2\ln(\tfrac{4}{\delta_c})},
\end{aligned}
\label{eq_26}
\end{equation}
where 
\begin{equation}
\begin{aligned}
\rho&=(M-(k-1)B)\text{, and }
\tau = \frac{8\eta^2k}{\beta^2 \lambda^2 \Delta^4} + \frac{\sigma_sk}{\sqrt{\pi}}\exp(-1/4\sigma_s^2) + \frac{k^2\mu^2}{(M/k-B)^2}.
\end{aligned}
\end{equation}
\end{proof}

\subsection{Overall Convergence of \randname over $T$ Rounds}
\label{app:corollary}
\begin{corollary}
\label{conv_cor}
    Suppose Assumption~\ref{ass:convex_smooth}-\ref{assum:seperation} hold, after $T=\frac{1}{K}\log\left(\frac{\Delta\sqrt{\lambda}}{2\epsilon_T\sqrt{L}}\right) + \frac{\log\left(\frac{\Delta\sqrt{\lambda}}{8(\frac{1}{2}-\alpha_0)\sqrt{L}}\right)}{\log(1-K)} $ rounds, for all cluster $j \in [k]$, we have at least probability $(1-\delta_c')$, $\|\theta_j^T - \theta_j^*\| \leq \epsilon_T$, where
    $$
        \epsilon_T = \frac{2\epsilon}{K}, ~
        K = \frac{\gamma L (B-2\tau M / \delta_c)}{2\rho}, ~\delta_c' = kT\delta_c. 
    $$
\end{corollary}
\begin{proof}
    Recall that the error floor of Theorem \ref{thm:one_step} $\epsilon$ is 
\begin{equation}
\epsilon = \frac{4 v}{\sqrt{\delta_c(B\delta_c-4\tau M)}}
+\frac{6 \tau \gamma L\Delta M}{\delta_c B}
+\frac{8\gamma v k\sqrt{\tau k M}}{B\delta_c\sqrt{\delta_c}}
+\frac{2\gamma\sigma_\theta C_\theta}{B}\sqrt{d + 2\sqrt{d\ln(\tfrac{4}{\delta_c})} + 2\ln(\tfrac{4}{\delta_c})}.
\end{equation}
 Let $K$ be the subtraction factor $\frac{\gamma L (B-2\tau M / \delta_c)}{2\rho}$ in Eq.~\eqref{eq_26}. We rewrite Eq.~\eqref{eq_26} as 
\begin{equation}
    \|\theta^+_1-\theta^*_1\|\leq \left(1-K\right)\|\theta_1 - \theta^*_1\| + \epsilon = \|\theta_1 - \theta^*_1\| + (\epsilon - K\|\theta_1 - \theta^*_1\|). 
\end{equation}
Given the bound of $\Delta$ in Assumption~\ref{assum:seperation}, we can ensure $\epsilon \leq K(\frac{1}{2}-\alpha_0)\Delta$, thus we have that the proposed method is contractive, that is, $\|\theta^{t+1}_1-\theta^*_1\| \leq |\theta^{t}_1-\theta^*_1\|$.

In the $t$-th round, assume we have
$ \|\theta^{t}_1-\theta^*_1\| \leq (\frac{1}{2}-\alpha_t)\sqrt{\frac{\lambda}{L}}\Delta, $ where the sequence of $\alpha_t$ should be non-decreasing due to its contractive convergence.

Suppose after $T'$ rounds we have $\alpha_t \leq \frac{1}{4}$, this can be achieved if both 
\begin{equation}
\label{eq:t_prime}
(1-K)^{T'}(\frac{1}{2}-\alpha_0)\Delta\leq \frac{1}{8}\sqrt{\frac{\lambda}{L}}\Delta~\text{, and}~\frac{1}{K}\epsilon \leq \frac{1}{8}\sqrt{\frac{\lambda}{L}}\Delta
\end{equation} 
hold, where the latter equation holds given bounds on $\Delta$ in Assumption \ref{assum:seperation}. 
Solving Eq.~\eqref{eq:t_prime}, we can get
\begin{equation}
    T' \geq \frac{\log\left(\frac{\Delta\sqrt{\lambda}}{8(\frac{1}{2}-\alpha_0)\sqrt{L}}\right)}{\log(1-K)},
\end{equation}
which means that after $T'$ rounds, $\|\theta^{T'}_1-\theta^*_1\| \leq \frac{1}{4}\sqrt{\frac{\lambda}{L}}\Delta$.

Let $T = T' + T''$, after $T$ rounds in total (with another $T''$ rounds), we have probability at least $1-kT\delta_c$, for all $j\in[k]$
\begin{equation}
    \|\theta^{T}_1-\theta^*_1\| \leq (1-K)^{T''}|\theta^{T'}_1-\theta^*_1\| + \frac{1}{K}\epsilon, 
\end{equation}
where the last term $\frac{\epsilon}{K}$ is calculated by the sum of series of $\epsilon + (1-K)\epsilon + (1-K)^2\epsilon + \cdots$. 

Let $T'' \geq \frac{1}{K}\log\left(\frac{k\Delta\sqrt{\lambda}}{4\epsilon\sqrt{L}}\right)$, we can bound $(1-K)^{T''}\|\theta^{T'}_1-\theta^*_1\|$ as
\begin{equation}
    (1-K)^{T''}\|\theta^{T'}_1-\theta^*_1\| \leq e^{-KT''}\cdot\frac{1}{4}\sqrt{\frac{\lambda}{L}}\Delta \leq \frac{1}{K}\epsilon, 
\end{equation}
which means $\|\theta^{T}_1-\theta^*_1\| \leq \frac{2\epsilon}{K} = \epsilon_T$
after $T = 
\frac{1}{K}\log\left(\frac{K\Delta\sqrt{\lambda}}{4\epsilon\sqrt{L}}\right)
+
\frac{\log\left(\frac{\Delta\sqrt{\lambda}}{8(\frac{1}{2}-\alpha_0)\sqrt{L}}\right)}{\log(1-K)} 
 = \frac{1}{K}\log\left(\frac{\Delta\sqrt{\lambda}}{2\epsilon_T\sqrt{L}}\right)
+
\frac{\log\left(\frac{\Delta\sqrt{\lambda}}{8(\frac{1}{2}-\alpha_0)\sqrt{L}}\right)}{\log(1-K)} 
$ 
rounds, with probability at least $(1-\delta_c')$ where $\delta_c' = kT\delta_c$.

\end{proof}

\section{Experimental Details} \label{app:exp_details}

\paragraph{Models.} For both image classification tasks, we employed a CNN network, and for the Shakespeare classification task, we used a LSTM network for classification. For the synthetic data, we used a linear regression model with slope and interscet as trainable parameters.

\paragraph{Hyperparameters.}
\label{app:hyperparam}
For the clipping bound of model updates $C_\theta$, we conduct grid search in \texttt{np.logspace(-1,-3,5)} and select the value based on model accuracy on validation set for all methods. 
For our hyperparameter $B$ (minimal cluster size), if not specified, we use grid search $B$ in $\{4,8,12\}$ by default in all experiments.
For the clipping threshold of cluster identifiers $s$, considering that all $s$ are one-hot vectors, the clipping threshold cannot affect cluster assignment. Since $\widetilde{s} = s/\max(1, \frac{s}{C_s}) + \mathcal{N}(0,C_s^2) = C_s\cdot(s+\mathcal{N}(0,1))$ if we select $C_s \leq 1$, and that server assign clusters based on the position of $\max\{\widetilde{s_j}\}_{j\in k}$, so it's only the position of $\max\{\widetilde{s_j}\}_{j\in k}$ that matters, which is irrelevant to $C_s$. In experiments, we use $0.1$ as the $C_s$. 
We tune local client-side learning rate from \{1e-6, 1e-5, 1e-4, 1e-3\}, and the number of local epochs from  \{5, 10, 15, 20\} on DP-FedAvg~\cite{iclr18_dpfedavg} and use the same values for all methods.

\paragraph{Hardware.}
Most experiments are conducted using 2 NVIDIA L40 GPU on a AMD Thread-Ripper HEDT platform in a desktop.

\section{Additional Experimental Results}


\subsection{Experiments on CIFAR10 and CIFAR100}
\label{app:cifar}
We further evaluate our method \randname on the CIFAR10 and CIFAR100 datasets \citep{cifar}. For both datasets, we partition the dataset in the same method as "Balanced Clusters" setup for FashionMNIST described in Section~\ref{sec:fashionmnist_text}. We compare \randname with various DP federated learning and federated clustering baselines. The experiments are conducted with privacy budgets $\varepsilon=4$ and $\varepsilon=8$.
We present classification accuracies in Table~\ref{table:acc_cifar}. As shown, RR-Cluster consistently outperforms the baseline methods on both CIFAR10 and CIFAR100 datasets.

\begin{table}[h]
\centering
\vspace{5pt}
\renewcommand\arraystretch{1.1}
\begin{tabular}{l|cc|cc}
\thickhline
& \multicolumn{2}{c|}{CIFAR10} & \multicolumn{2}{c}{CIFAR100} \\
\multirow{-2}{*}{Methods~~~} & \textcolor{mygreen}{$\varepsilon=4$} & \textcolor{mygreen}{$\varepsilon=8$} & \textcolor{mygreen}{$\varepsilon=4$} & \textcolor{mygreen}{$\varepsilon=8$} \\
\hline
DP-FedAvg & 52.89 & 54.48 & 14.05 & 14.49 \\
DP-FedProx & 52.04 & 54.61 & 14.48 & 14.53 \\
DP-FeSEM & 53.01 & 55.87 & 16.23 & 17.08 \\
DP-FedCAM & 52.50 & 54.73 & 16.36 & 17.04 \\
DP-IFCA & 53.91 & 56.77 & 16.29 & 17.47 \\
\hline
\randname~(IFCA) & \textbf{57.46} & \textbf{58.41} & \textbf{17.83} & \textbf{18.00} \\
\thickhline
\end{tabular}
\vspace{0.5em}
\caption{\small Comparison with baselines on CIFAR10 and CIFAR100 datasets. \randname achieves the highest accuracies.}
\label{table:acc_cifar}
\end{table}

\subsection{Full results with Federated Personalization Methods}
\label{app:tab_with_per}

\begin{table*}[!h]
\centering
    \resizebox{0.98\linewidth}{!}{
    \renewcommand\arraystretch{1.1}
        \begin{tabular}{l|ccc|ccc}
        \thickhline
		& \multicolumn{3}{c|}{Balanced Clusters}&\multicolumn{3}{c}{Imbalanced Clusters} \\
      	\multirow{-2}{*}{Methods~~~} & \textcolor{mygreen}{$\varepsilon=2$}  & \textcolor{mygreen}{$\varepsilon=4$} & \textcolor{mygreen}{$\varepsilon=8$}  & \textcolor{mygreen}{$\varepsilon=2$}  & \textcolor{mygreen}{$\varepsilon=4$} & \textcolor{mygreen}{$\varepsilon=8$}   \\
    \hline
    DP-FedAvg & 60.88 & 61.50 & 62.72 & 60.80 & 61.53 & 61.59 \\
    \hline
    DP-FedPer & 61.03 & 61.78 & 62.19 & 60.42 & 61.83 & 62.35 \\
    
    \hline
    DP-IFCA & 62.68 & 64.46 & 65.35 & 56.12 & 58.05 & 59.63 \\
    DP-FeSEM & 63.16 & 64.68 & 64.76 & 58.55 & 59.82 & 60.10 \\
    DP-FedCAM & 49.70 & 56.86 & 61.54 & 43.98 & 47.51 & 50.35 \\

    \hline
   \randname (IFCA) &\reshl{65.69}{3.01}&\reshl{66.63}{2.17}&\reshl{67.51}{2.16}& \reshl{61.31}{5.19}&\reshl{61.99}{3.94}&\reshl{63.77}{4.14}\\
    \randname (FeSEM) &\reshl{63.89}{0.73}&\reshl{64.80}{0.12}&\reshl{64.86}{0.10}& \reshl{58.70}{0.15}&\reshl{60.34}{0.52}&\reshl{61.16}{1.06}\\
    \randname (FedCAM) &\reshl{53.91}{4.21}&\reshl{58.65}{1.79}&\reshl{64.78}{3.24}& \reshl{46.12}{2.14}&\reshl{47.95}{0.44}&\reshl{50.78}{0.43}\\
    \hline 
    \end{tabular}}
    \vspace{0.5em}
    \caption{\small Full version of Table~\ref{table:acc_fashionmnist}. Comparison with baselines on {FashionMNIST}. See discussions in Section~\ref{sec:exp_sota}.}
    \vspace{1em}
    \label{app:acc_fashionmnist}
\end{table*}

\begin{table*}[!h]
\centering
    \resizebox{0.98\linewidth}{!}{
		\renewcommand\arraystretch{1.1}
        \begin{tabular}{l|ccc|ccc}
		\thickhline
		& \multicolumn{3}{c|}{Mild Client Heterogeneity}&\multicolumn{3}{c}{High Client Heterogeneity} \\
        \multirow{-2}{*}{Methods~~~} & \textcolor{darkgreen}{$\varepsilon=2$}  & \textcolor{darkgreen}{$\varepsilon=4$}  & \textcolor{darkgreen}{$\varepsilon=8$}  & \textcolor{darkgreen}{$\varepsilon=2$}  & \textcolor{darkgreen}{$\varepsilon=4$}  & \textcolor{darkgreen}{$\varepsilon=8$}      \\
    \hline
    DP-FedAvg & 26.47 & 30.95 & 32.61 & 23.82 & 24.46 & 25.54 \\
    \hline
    DP-FedPer & 40.67 & 40.81 & 41.23 & 24.78 & 25.51 & 26.12 \\
    \hline
    DP-IFCA & 62.23 & 65.39 & 65.93 & 32.53 & 35.60 & 36.19 \\
    DP-FeSEM & 60.43 & 65.72 & 66.24 & 31.12 & 37.52 & 38.74 \\
    DP-FedCAM & 57.35 & 57.22 & 61.26 & 23.66 & 24.55 & 29.70 \\
    \hline
    \randname (IFCA) & \reshl{66.38}{4.15} & \reshl{66.75}{1.03} & \reshl{67.04}{0.80} & \reshl{37.13}{4.60} & \reshl{38.92}{1.40} & \reshl{39.05}{0.31} \\
    \hline
    \end{tabular}}
    \vspace{0.5em}
    \caption{\small Full version of Table~\ref{table:acc_emnist}. Comparison with baselines on {EMNIST}. See discussions in Section~\ref{sec:exp_sota}.}
    \vspace{1em}
    \label{app:acc_emnist}
\end{table*}

\begin{table}[!h]
  \centering
    \resizebox{0.5\textwidth}{!}{
		\renewcommand\arraystretch{1.1}
        \begin{tabular}{l|cccc}
		\thickhline
   {Methods~~~} & \textcolor{darkgreen}{$\varepsilon=4$}  & \textcolor{darkgreen}{$\varepsilon=8$} & \textcolor{darkgreen}{$\varepsilon=16$}  \\
    \thickhline
    DP-FedAvg & 04.47 & 13.43 & 17.53 \\
    \hline
    DP-FedPer & 13.09 & 13.14 & 13.18 \\
    \hline
    DP-IFCA & 12.62 & 12.64 & 13.20 \\
    DP-FeSEM & 10.97 & 12.99 & 15.89 \\
    DP-FedCAM & - & - & 04.47 \\
    \hline
    \texttt{RR-Cluster(IFCA)} &{13.42}&{13.83}&{16.25}\\
    \hline
    \end{tabular}}
    \vspace{0.5em}
\caption{\small Full version of Table~\ref{table:acc_shakespeare}. Comparison with baselines on {Shakespeare}. See discussions in Section~\ref{sec:exp_sota}.}
\label{app:acc_shakespeare}
\end{table}

In this section, we present the full versions of the experimental results tables 
discussed in Section~\ref{sec:exp_sota} of the main paper. For detailed analysis and discussion of these findings, please refer to Section~\ref{sec:exp_sota}.

\subsection{Experiments on Synthetic data}
\label{app:side_effect}

To further demonstrate that \randname has the side effect of mitigating model collapse in non-private settings, we conduct experiments on a synthetic dataset. 
Each synthetic dataset pair $({x}, \hat{y})$ with $x, \hat{y} \in \mathbb{R}$ is generated by $\hat{y} = k{x} + b + \epsilon$, where $\epsilon \sim \mathcal{N}(0, \sigma^2)$. The trainable parameters are $k, b \in \mathbb{R}$, and we use the MSE loss. We create four scenarios based on whether the clusters are balanced, and whether there is a clear separation between clusters.

When the truth number of clusters is large (first row), we see that  IFCA  suffers from model collapse even if we set $k=4$. In contrast, our \randname efficiently trained all four cluster models. We also note that in the other scenario where the clusters are imbalanced (so that it is more likely to dynamically rebalance the size of clusters) (second row), our method can introduce some additional bias compared with IFCA.
Further, we extend the original data distribution (the `Easy to Cluster') to `Hard to Cluster' settings by a increased Gaussian variance on client data generation. We show that our \randname works well with handling data that harder to cluster, and that our methods also demonstrate strong robustness against model collapse, since we have four cluster model fully aligned with the object data, while IFCA have one cluster model left at the bottom and remained untrained.


\vspace{1em}
\begin{table*}[!h]
\setlength{\intextsep}{0pt}
\setlength{\textfloatsep}{0pt}
\setlength{\tabcolsep}{0.5pt}
\centering
\begin{tabular}{cccc}
        Data & FedAvg & IFCA & \randname (IFCA) \\


\makecell[c]{Balance + \\ Easy to cluster} & 
\begin{tabular}{c}
\includegraphics[width=0.27\textwidth]{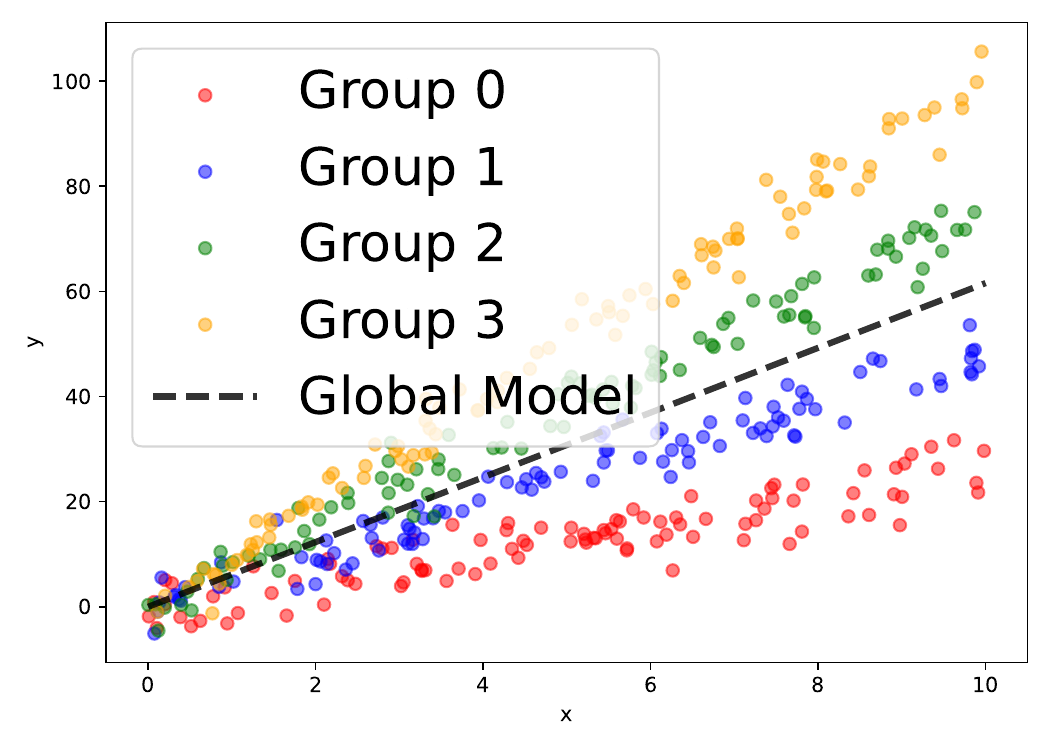}
\end{tabular} &

\begin{tabular}{c}
\includegraphics[width=0.27\textwidth]{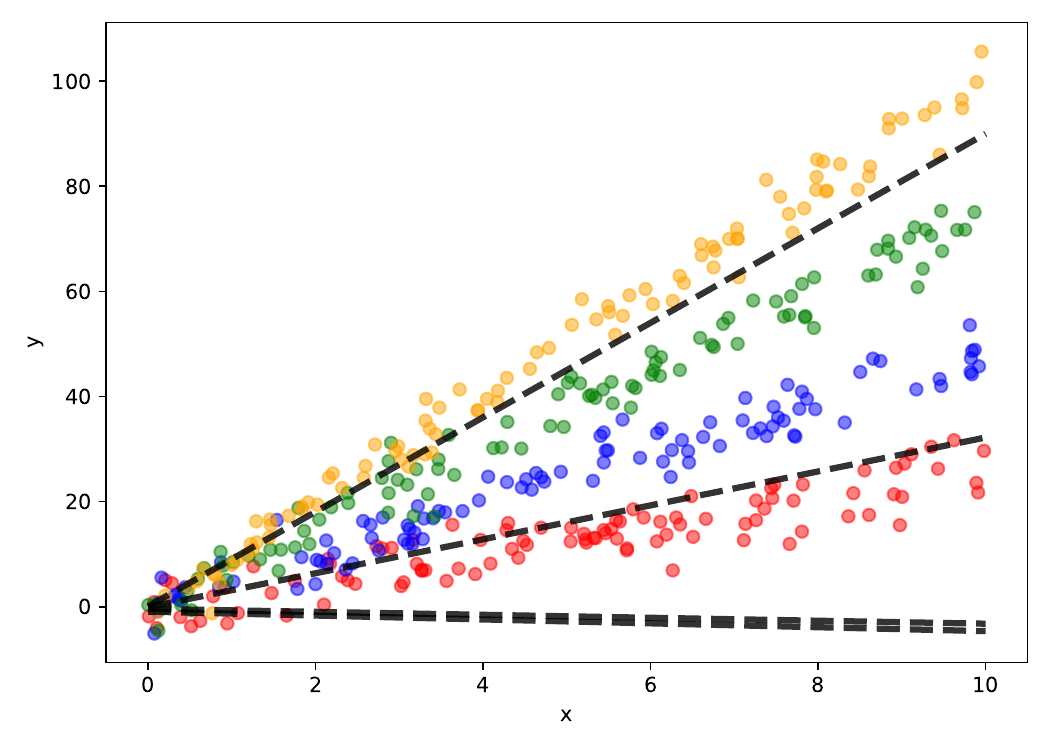}
\end{tabular} &

\begin{tabular}{c}
\includegraphics[width=0.27\textwidth]{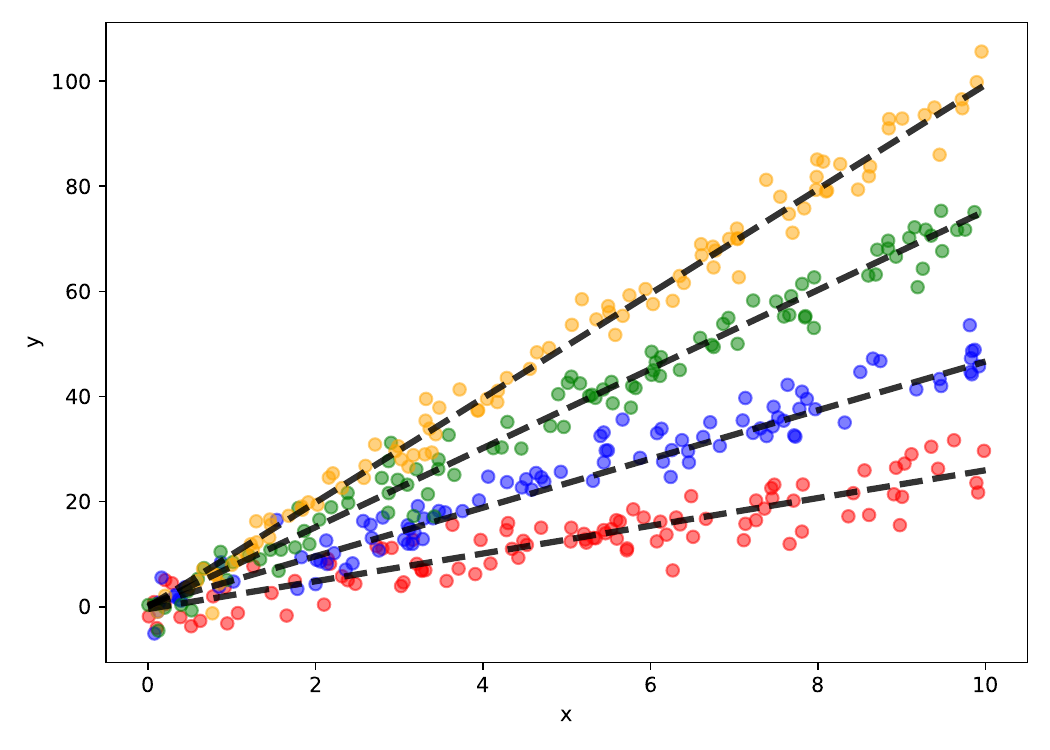}
\end{tabular}
\\

\makecell[c]{Balance + \\ Hard to cluster} & 
\begin{tabular}{c}
\includegraphics[width=0.27\textwidth]{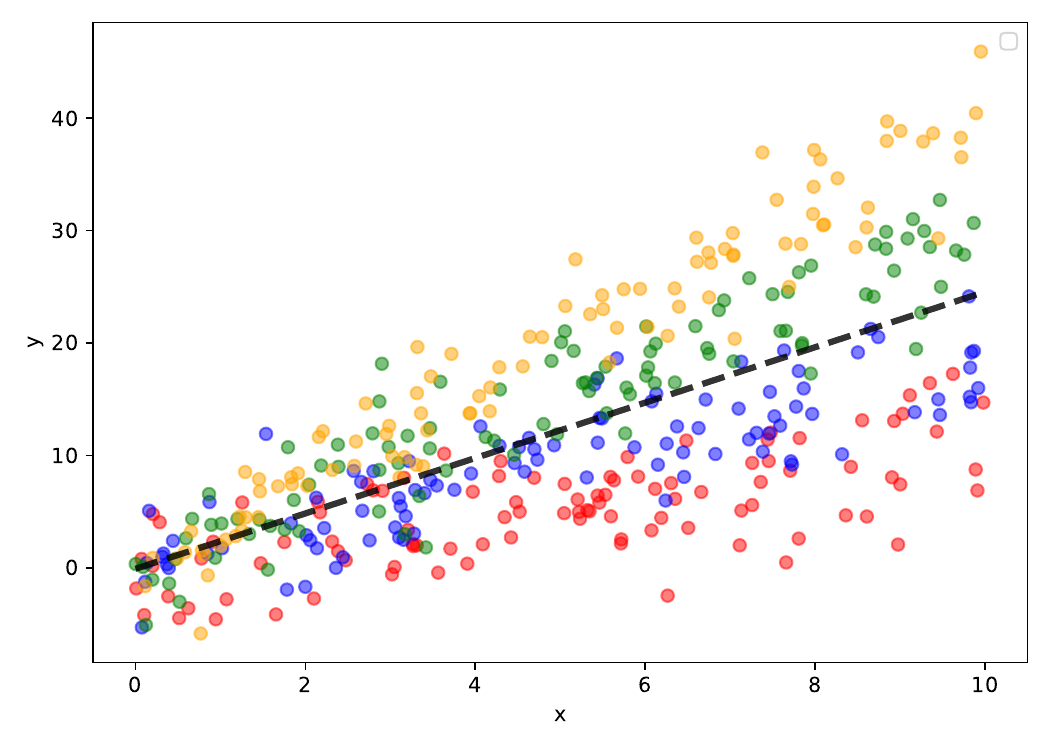}
\end{tabular} &

\begin{tabular}{c}
\includegraphics[width=0.27\textwidth]{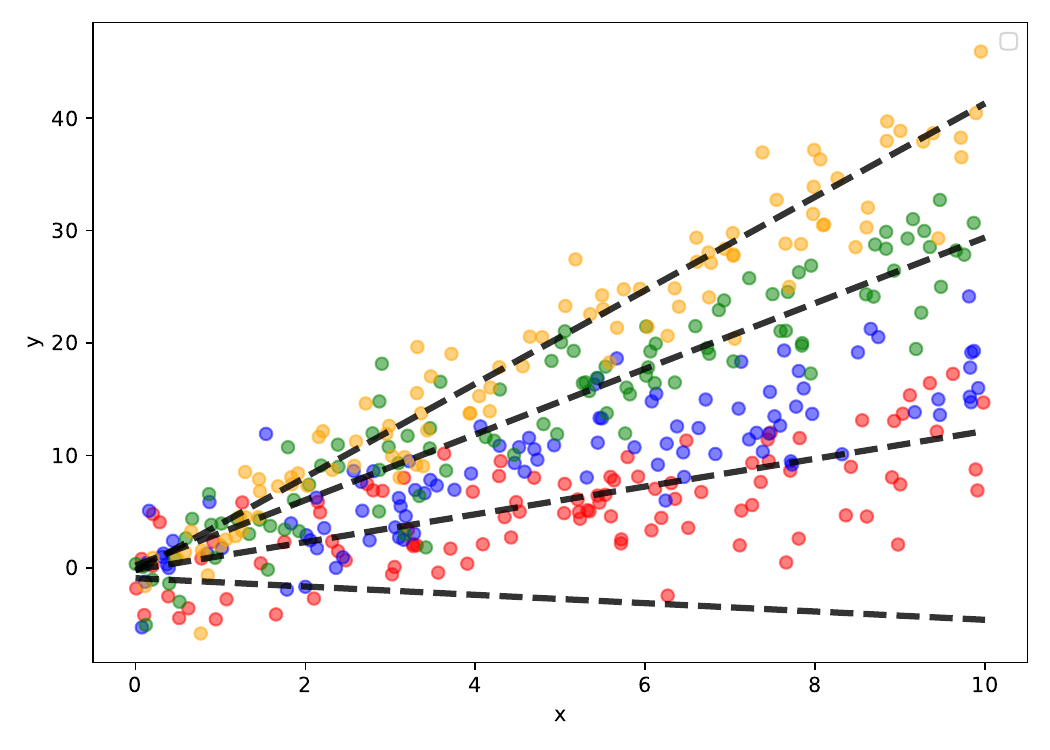}
\end{tabular} &

\begin{tabular}{c}
\includegraphics[width=0.27\textwidth]{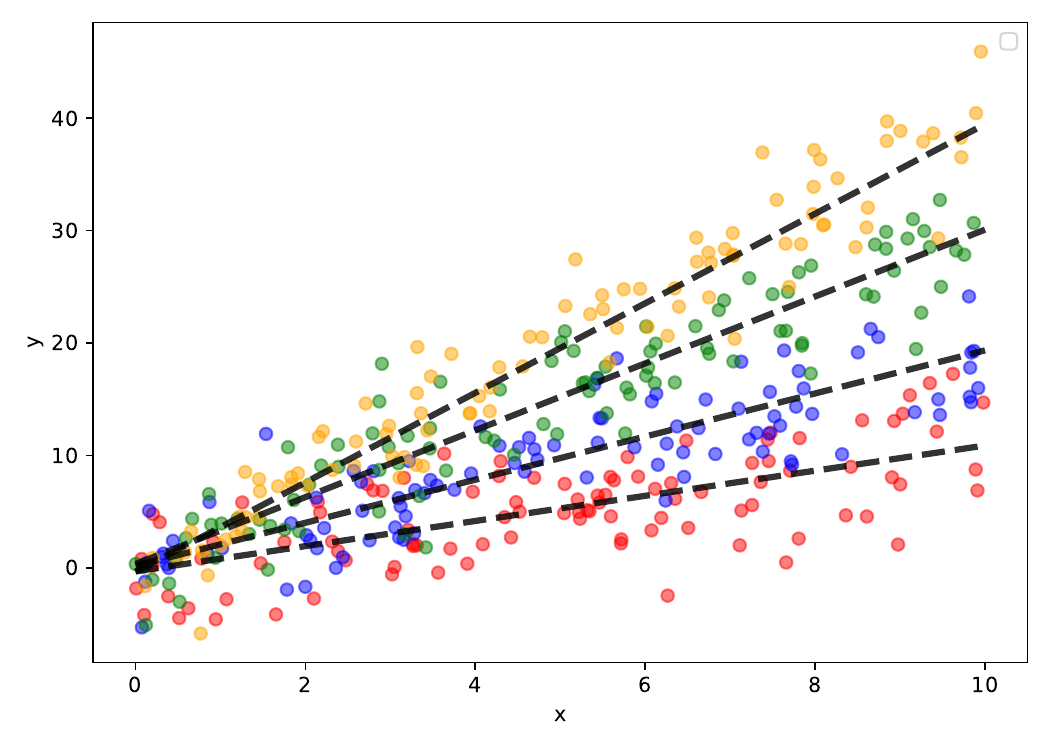}
\end{tabular}
\\

\makecell[c]{Imbalance + \\ Easy to cluster} & 
\begin{tabular}{c}
\includegraphics[width=0.27\textwidth]{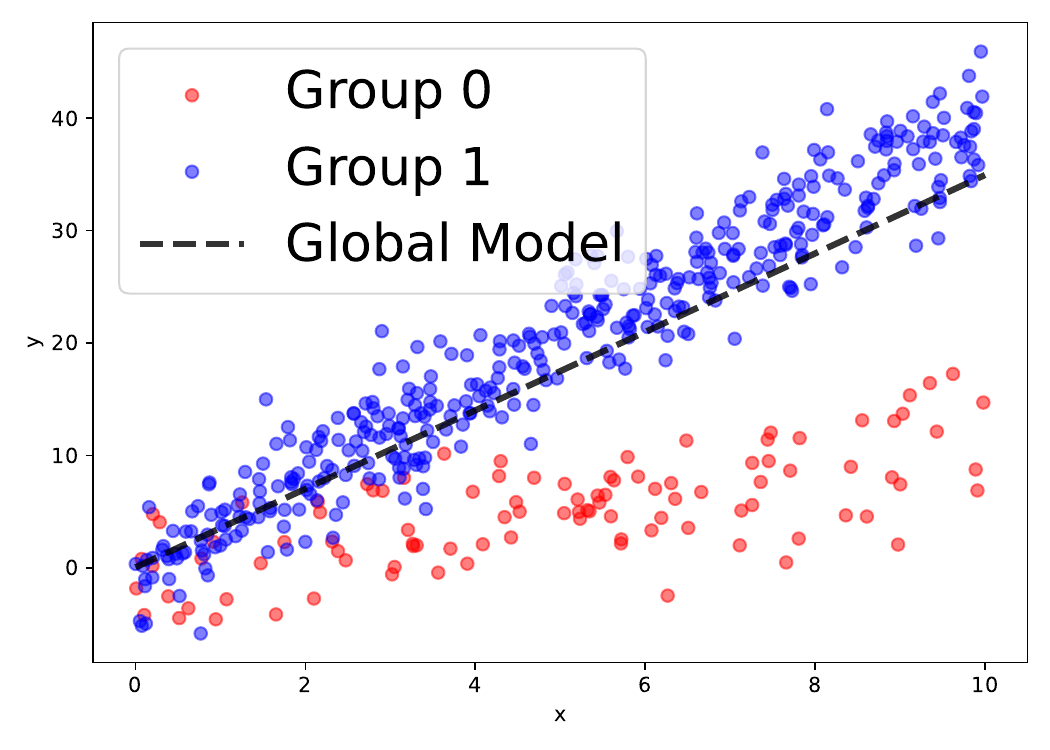}
\end{tabular} &

\begin{tabular}{c}
\includegraphics[width=0.27\textwidth]{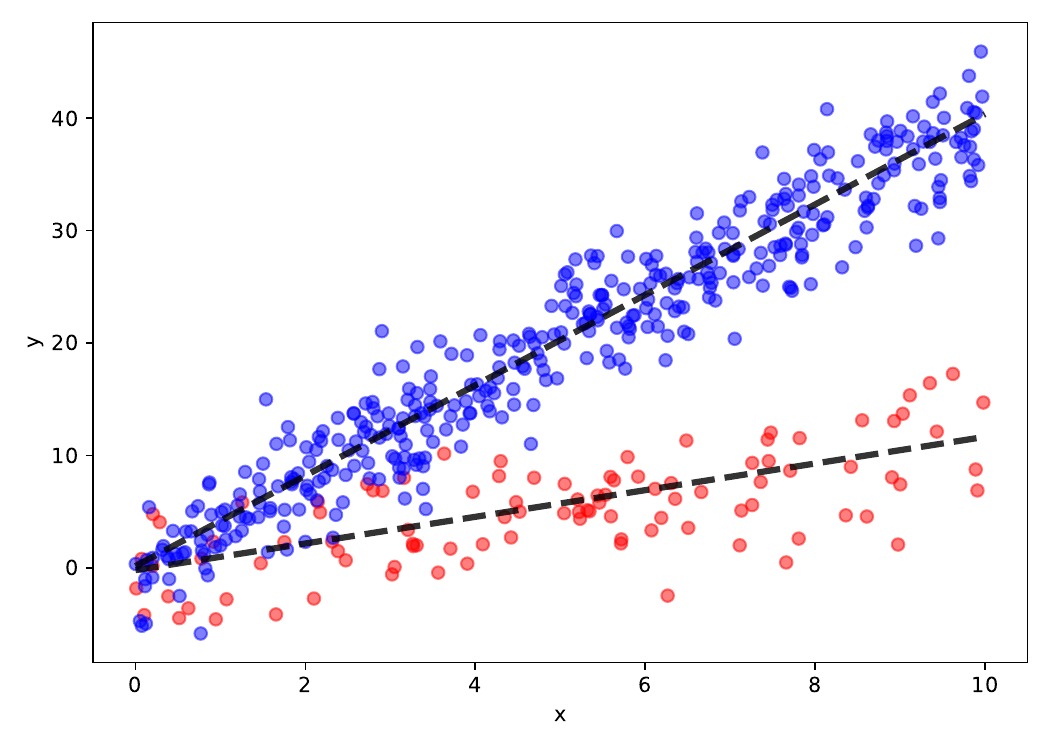}
\end{tabular} &

\begin{tabular}{c}
\includegraphics[width=0.27\textwidth]{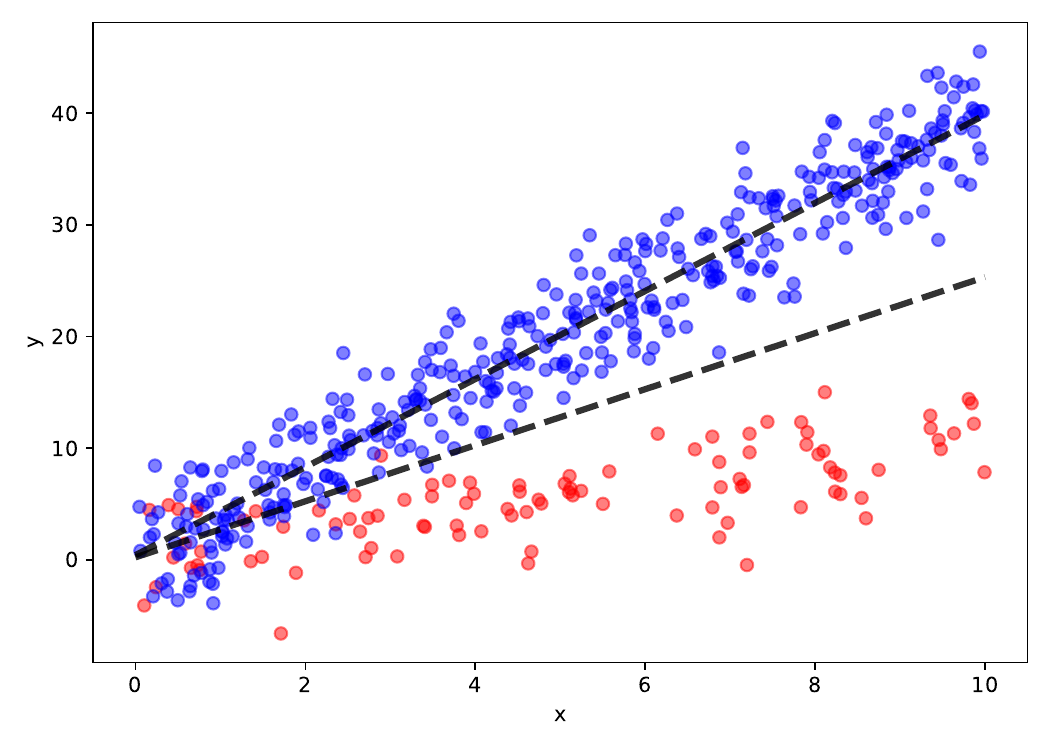}
\end{tabular} 
\\

\makecell[c]{Imbalance + \\ Hard to cluster} & 
\begin{tabular}{c}
\includegraphics[width=0.27\textwidth]{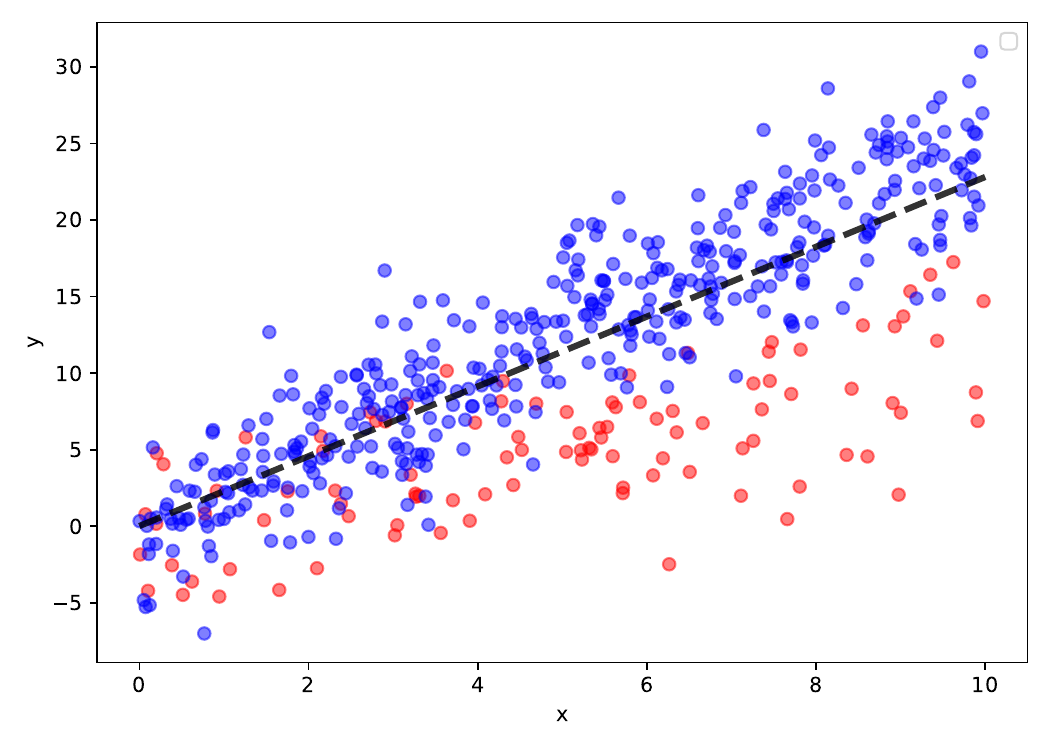}
\end{tabular} &

\begin{tabular}{c}
\includegraphics[width=0.27\textwidth]{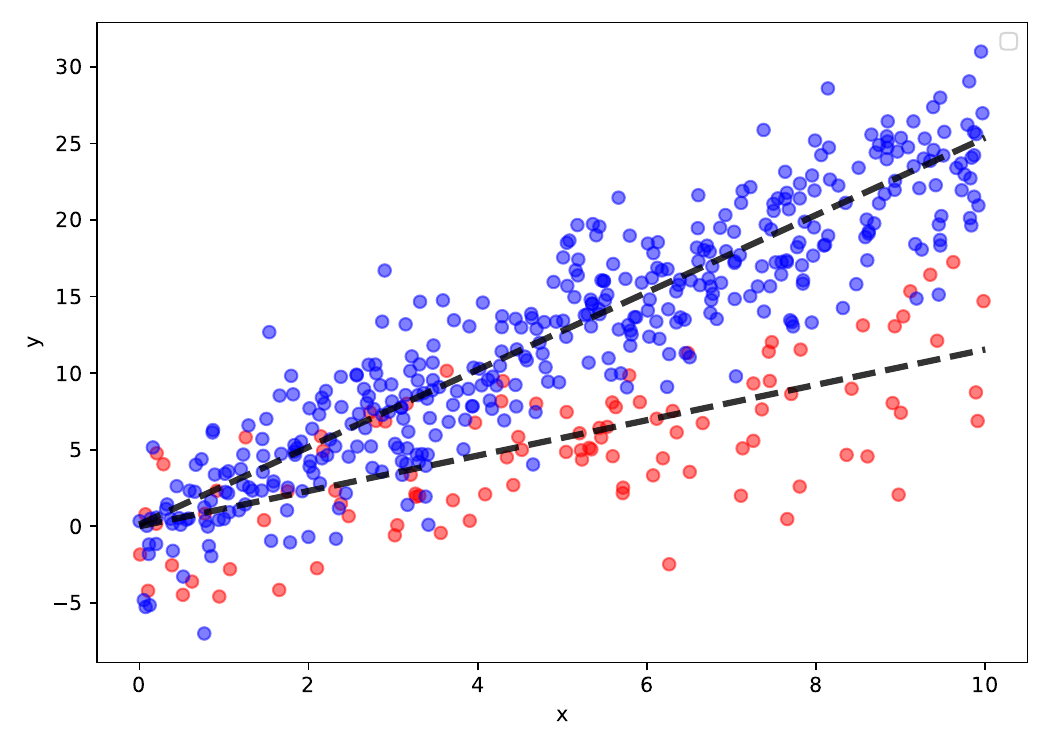}
\end{tabular} &

\begin{tabular}{c}
\includegraphics[width=0.27\textwidth]{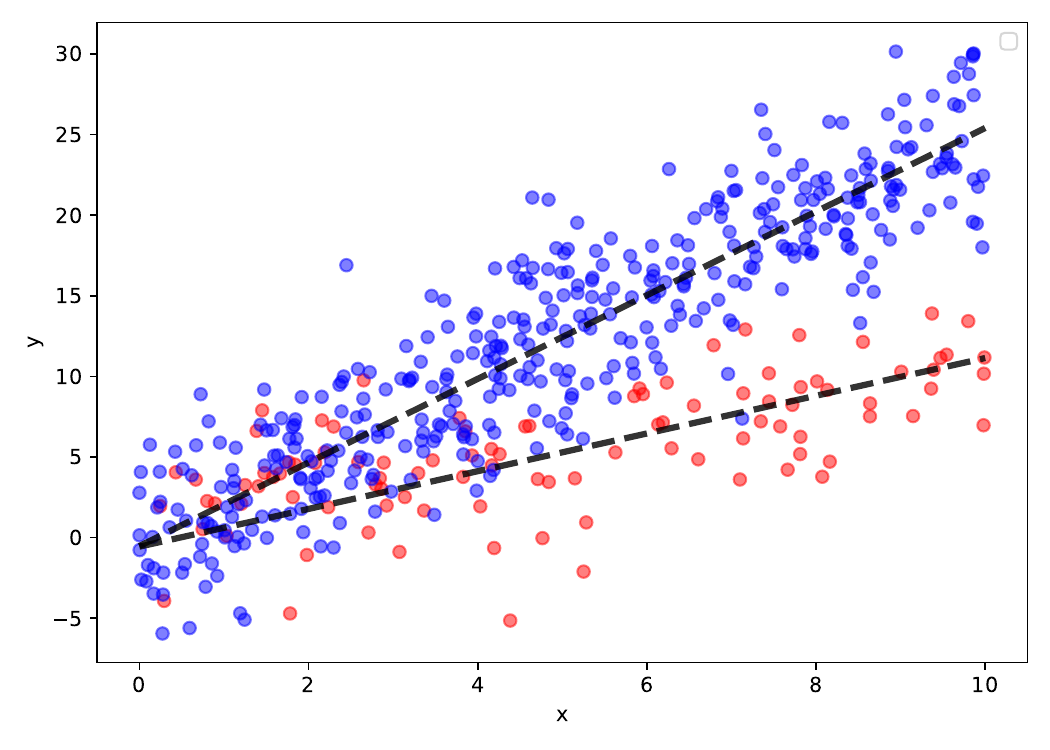}
\end{tabular} 
\\


\end{tabular}
\caption{To understand the effects of \randname in non-private settings, we simulate the balanced and imbalanced distribution of clients on the synthetic dataset and compare with other methods. The original datapoints (data generated from different noisy line are in different colors) is shown in each pictures. We plot the final global/cluster model as black dashlines in each picture.}
\label{tab:synthetic_additional}
\end{table*}

\end{document}